\documentclass{article}

\usepackage{geometry}
\usepackage[round]{natbib}

\usepackage{indentfirst}
\usepackage{amsmath, amssymb, amsthm, amsfonts}
\usepackage{xcolor}
\usepackage{algorithm}
\usepackage{algorithmic}
\usepackage{paralist}
\usepackage{xurl}
\usepackage{graphicx}
\usepackage{caption}
\usepackage{subcaption}
\usepackage{diagbox}

\theoremstyle{plain}

\newtheorem{lemma}{Lemma}

\theoremstyle{definition}
\newtheorem{definition}{Definition}
\newtheorem{problem}{Problem}
\newtheorem{example}{Example}

\definecolor{colorx}{rgb}{0.4, 0.1, 0.7}

\newcommand{\eps}{{\varepsilon}}
\newcommand{\Greedy}{\textsc{Greedy}}
\newcommand{\Ada}{\textsc{Ada}}
\newcommand{\LR}{\textsc{LR}}
\newcommand{\PF}{\textsc{PF}}
\newcommand{\hPF}{\textsc{hPF}}
\newcommand{\ERM}{\textsc{ERM}}
\newcommand{\BeFair}{\textsc{BeFair}}
\newcommand{\G}{\textsc{G}}

\newcommand{\mae}{\mathrm{MAE}}

\usepackage{thmtools}
\usepackage{thm-restate}
\usepackage{hyperref}
\usepackage[capitalise]{cleveref}

\usepackage[symbol]{footmisc}
\usepackage{wrapfig}

\usepackage{multibib}
\newcites{SM}{Supplementary References}

\usepackage{bbm}
\newcommand{\err}{{\mathrm{err}}}
\newcommand{\indicator}{\mathbbm{1}}
\newcommand{\E}{\mathbbm{E}}
\newcommand{\gro}{\mathcal{G}}
\newcommand{\hyp}{\mathcal{H}}
\newcommand{\hpr}{{h^\prime}}

\usepackage[inline]{enumitem}

\title{Fair for All: Best-effort Fairness Guarantees for Classification}

\usepackage{authblk}

\author[1]{Anilesh K. Krishnaswamy}
\author[2]{Zhihao Jiang}
\author[1]{Kangning Wang}
\author[3]{Yu Cheng}
\author[1]{Kamesh Munagala}

\affil[1]{Duke University}
\affil[2]{Tsinghua University}
\affil[3]{University of Illinois at Chicago}

\date{}

\begin{document}

%

%

\maketitle

\begin{abstract}

Standard approaches to group-based notions of fairness, such as \emph{parity} and \emph{equalized odds}, try to equalize absolute measures of performance across known groups (based on race, gender, etc.). Consequently, a group that is inherently harder to classify may hold back the performance on other groups; and no guarantees can be provided for unforeseen groups. Instead, we propose a fairness notion whose guarantee, on each group $g$ in a class $\mathcal{G}$, is relative to the performance of the best classifier on $g$. We apply this notion to broad classes of groups, in particular, where (a) $\mathcal{G}$ consists of all possible groups (subsets) in the data, and (b) $\mathcal{G}$ is more streamlined.

For the first setting, which is akin to groups being completely unknown, we devise the {\sc PF} (Proportional Fairness) classifier, which guarantees, on any possible group $g$, an accuracy that is proportional to that of the optimal classifier for $g$, scaled by the relative size of $g$ in the data set. Due to including all possible groups, some of which could be too complex to be relevant, the worst-case theoretical guarantees here have to be proportionally weaker for smaller subsets.

For the second setting, we devise the  {\sc BeFair} (Best-effort Fair) framework which seeks an accuracy, on every $g \in \mathcal{G}$, which approximates that of the optimal classifier on $g$, independent of the size of $g$. Aiming for such a guarantee results in a non-convex problem, and we design novel techniques to get around this difficulty when $\mathcal{G}$ is the set of linear hypotheses. We test our algorithms on real-world data sets, and present interesting comparative insights on their performance.

\end{abstract}

\section{Introduction}\label{sec:intro}



Machine learning is playing an ever-increasing role in making decisions that have a significant impact on our lives. Of late, we have seen the deployment of machine learning methods to provide advice for decisions pertaining to criminal justice \citep{angwin2016machine, berk2018fairness}, credit/lending  \citep{koren2016does}, health/medicine \citep{rajkomar2018ensuring}, etc. Given the concerns of disparate impact and bias in this regard \citep{angwin2016machine, barocas2016big}, it is imperative that machine learning models are fair.

The question of defining notions of fairness, and developing methods to achieve them, has received a great deal of attention \citep{barocas2017fairness, binns2017fairness}. A common theme among the many approaches proposed thus far \citep{kleinberg2018inherent, chouldechova2017fair} is to fix beforehand a list of protected groups, and then ask for the (approximate) equality of some statistical measure across them. For example, \emph{parity} seeks to equalize the accuracy across the given groups \citep{calders2009building}, while \emph{equalized odds} seeks to equalize false positive or false negative rates \citep{hardt2016equality}.

Classical definitions of fairness from microeconomics have also found application in machine learning \citep{balcan2019envy, chen2019proportionally, HossainMS19}. In particular, there has been recent work \citep{zafar2017parity, ustun2019fairness} on adapting the notion of \emph{envy-freeness}, which is born out of fair division theory \citep{brams1996fair}, to a group-based variant tailored to (binary) classification -- every given pre-defined group should prefer the way it is classified (on aggregate) in comparison to how it would have been if it assumed the identity of some other group. 

A major drawback of the aforementioned approaches is that they aim for an absolute guarantee: when some of the groups are inherently harder to classify than others, trying to achieve a particular measure of fairness, say equalized odds \citep{hardt2016equality}, could do more harm than good by bringing down the accuracy on a group that is easier to classify (see Figure \ref{fig:harm} for an example). In this paper, we take a more relative \emph{best-effort} approach: aiming for guarantees that are defined in terms of how well each group can be classified in itself.

\begin{figure}
    \centering
    \includegraphics[scale=0.25]{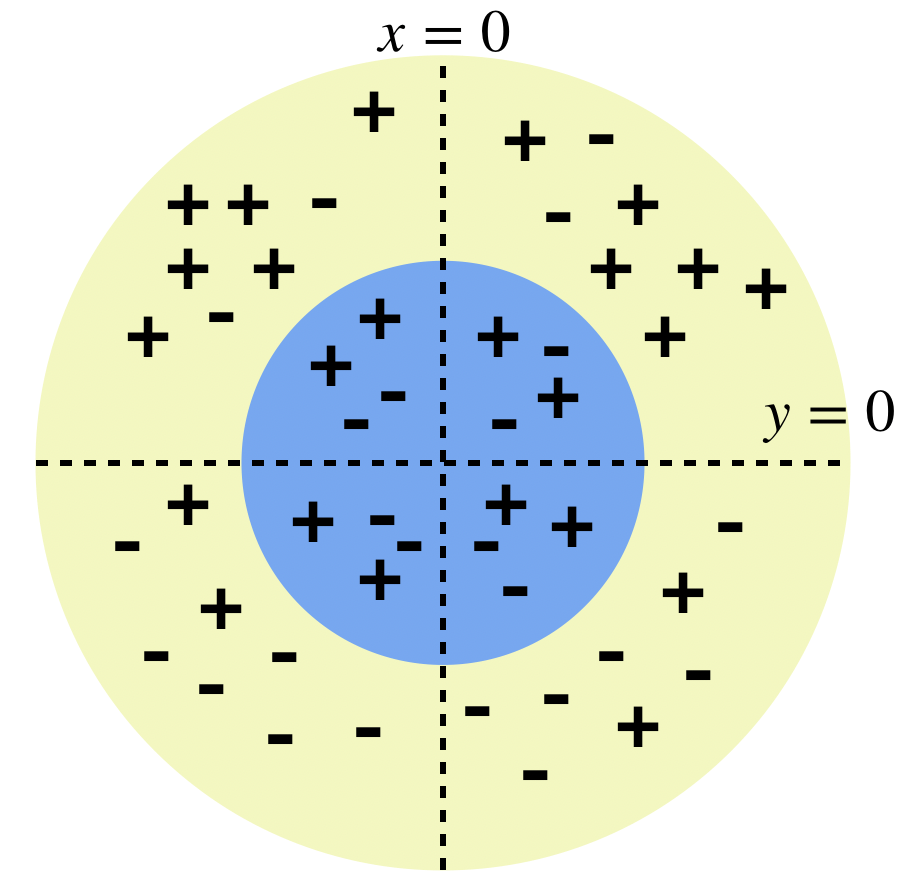}
    \caption{Given two groups Blue and Yellow (with true labels as shown), we have to choose just between the two classifiers $x=0$ and $y=0$. The Blue group is inherently harder to classify. Equalized odds makes us choose the classifier $x=0$, thereby hurting the Yellow group. We could choose $y=0$ with no aggregate effect on Blue, doing much better on Yellow. 
    } 
    \label{fig:harm}
\end{figure}

Another drawback of the standard approaches to fairness is that they depend critically on the specification of groups (via sensitive features such as race, gender, etc.). In many cases, the sensitive features are either missing \citep{chen2019fairness}, or unusable, considering the need to adhere to \emph{treatment parity} and anti-discrimination laws \citep{barocas2016big}. Even if they can be used, it is sometimes not clear what the right categorization within them should be. For instance, it could be that a particular demographic group, which is defined on the basis of a shared cultural or ethnic feature, is actually a collection of hidden subgroups that are otherwise quite heterogeneous in terms of other socio-economic indicators \citep{meier2012latino, chang2011debunking}. 
Therefore, mis-specifying or mis-calibrating the protected groups could end up hurting some groups within the data, potentially leading to unintended consequences such as a feeling of resentment among them \citep{hoggett2013fairness}.

We  use  the  following  instructive  albeit  stylized  example to illustrate the effect of missing group information.
\begin{example}
As shown below in Table \ref{tab:example1}, there are two binary features $a,b \in \{0,1\}$, and a hidden demographic feature $c \in \{0,1\}$. The target label $y$ follows the formula $y = (a \wedge c) \vee (b \wedge \neg c)$: if the hidden feature $c = 1$, then $a$ is a perfect classifier, and if $c=0$, then $b$ is a perfect classifier. For brevity, we define three groups: $P = \{(1,0,0), (0,1,0)\}$, $Q = \{(1,0,1), (0,1,1)\}$, and $R = (P \cup Q)^\mathsf{c}$.

As a concrete example, suppose each data point corresponds to a job candidate.
The hidden feature $c$ corresponds to \emph{gender}, $(a,b)$ correspond to measures of two different traits, and $y$ the assignment to one of two jobs. Suppose the family of available classifiers is $\mathcal{H} = \{a,b\}$. It can be seen that any of these classifiers does poorly in terms of fairness for groups based on the hidden feature $c$. 

Consider the classifier $a$ (in this case, a solution to the standard {Empirical Risk Minimization} ({\ERM}) with $0$-$1$ loss), which correctly classifies all data except those in $P$. Therefore, the {\ERM} classifier $a$ is unfair to those of gender $0$. Our {\PF} classifier (which we shall see later) gets around this issue by randomizing between $a$ and $b$. This allows us to classify $P$ and $Q$ correctly with probability $0.5$, and $R$ correctly with probability $1$. Note that the {\PF} classifier is able to treat every gender equally in expectation, even without access to the gender labels.
\end{example}

\begin{table}
\small
\centering
\begin{tabular}{c|cccccccc}
 $a$ & $1$&  $1$&  $0$&  $0$&  $1$&  $1$&  $0$&  $0$\\ 
 $b$ & $1$&  $0$&  $1$&  $0$&  $1$&  $0$&  $1$&  $0$\\
 $c$ & $1$&  $1$&  $1$&  $1$&  $0$&  $0$&  $0$&  $0$\\\hline
 $y$ & $1$&  $1$&  $0$&  $0$&  $1$&  $0$&  $1$&  $0$
 \smallskip
\end{tabular}
\caption{$a,b$ are visible features, $c$ is a hidden ``demographic'' feature, and $y$ is the target label.}
\label{tab:example1}
\end{table}

In order to deal with the above issues, we take a best-effort approach to fairness, one that can be applied to broad classes of groups. If the group identities were well defined and limited in number, simple solutions work: for example, one could perhaps train decoupled classifiers \citep{ustun2019fairness, dwork2018decoupled}. With the unavailability or mis-specification of group information, however, the problem is much more interesting. In this regard, we look at two settings -- where the groups taken into account are given by \begin{enumerate*}[label=(\alph*)]
    \item all possible subsets in the data, and
    \item a more streamlined class of groups, such as all linearly separable ones. 
\end{enumerate*}

In the former setting, we are effectively reasoning about fairness even though there are no pre-specified groups. Standard statistical notions are of no use in this regard, for, as noted by \citet{KearnsNRW18}, ``we cannot insist on any notion of statistical fairness for every subgroup of the population: for example, any imperfect classifier could be accused of being unfair to the subgroup of individuals defined ex-post as the set of individuals it mis-classified.''

In fact, such a limitation also applies to any deterministic classifier. Therefore, focusing on randomized classifiers, the questions we consider first (Section \ref{sec:pf}) are: \emph{What is the best possible best-effort guarantee that can be achieved for all groups simultaneously?} We will see that, on account of taking all groups into account, some of which could be too complex to be meaningful in practice, we have to settle for guarantees that are proportionally weaker for smaller subsets. \emph{Are there algorithms that achieve such a guarantee?} We answer this question in the affirmative by devising the Proportional Fairness (\PF) classifier.

The next natural question is: \emph{can we do better if we consider a more streamlined class of groups?} We will see (Section \ref{sec:pf_error}) that this is indeed the case. We note here that standard fairness notions (such as \emph{parity}) can also be applied in such settings, by effectively solving a convex optimization problem \citep{KearnsNRW18}. In our best-effort fairness (\BeFair) approach, even when we consider linearly separable groups, we need to solve a non-convex problem. A major contribution of our work is devising a way of dealing with this difficulty, and at that, one that works well in practice. In Section \ref{sec:experiments}, we \emph{evaluate all our algorithms on real-world datasets}. We see that our {\BeFair } approach is able to achieve strong best-effort guarantees, significantly better than standard {\ERM } classifiers. We also present several empirical insights on the performance on {\PF}, mostly in line with our theoretical results.

A more detailed overview of our results are provided in Section \ref{sec:model}. All our proofs are provided in the Appendix.

\subsection{Related literature}
\label{sec:related}
The extant literature on fairness in machine learning \cite{hardt2016equality, kamiran2012data, hajian2012methodology, chouldechova2017fair, corbett2017algorithmic} primarily considers statistical notions of fairness which require the protected groups to be specified as input to the (binary) classification problem. Many of these notions are further known to be incompatible with one another \cite{kleinberg2018inherent, friedler2016possibility}. Individual notions of fairness, which loosely translate to asking for ``similar individuals'' to be ``treated similarly'', have also been studied \cite{dwork2012fairness}. However, this requires additional assumptions to be made about the problem at hand, in the form of, e.g., a ``similarity metric'' defined on pairs of data points. Another related notion is envy-freeness \citep{HossainMS19}, which isn't very useful without group information (more in the Appendix).

There are several papers on fairness that utilize the broadly applicable framework of minimax optimization in their algorithms \citep{agarwal2018reductions, rezaei2020fairness, baharlouei2019renyi, madras2018learning}.
For example, \cite{rezaei2020fairness} derive a novel distributionally robust classification method by incorporating fairness criteria into a worst-case loss minimization program. When compared with this literature, our work is different in one or both of two senses: First, we have a novel and conceptually different notion of fairness that approximates the best-effort guarantee for each group, as opposed to objectives such as max-min fairness or parity. Second, the cardinality of the set of groups can be unbounded (defined by linear constraints on either the feature space or its basis expansion) in our case, as opposed to operating with a fixed set of groups. As we shall see later, each of these aspects presents its own technical challenges.

Several issues have been raised with respect to defining the demographic groups that need to be considered for fairness. \cite{chen2019fairness} assess the prevalence of disparity when missing demographic identities are imputed from the data. \cite{hashimoto2018fairness} look at a model where user retention among different groups is linked to the accuracy achieved on them respectively, and design algorithms that improve the user retention among minority groups based on distributionally robust optimization. Their methods, while oblivious to the identity of the groups, operate under the assumption that there are a fixed number $K$ of groups, and work well in practice for small $K$. \citep{kim2019multiaccuracy} develop multi-accuracy auditing to guarantee the fairness for identifiable subgroups, by post-processing the classifier such that it is unbiased. \cite{KearnsNRW18} study the problem of auditing classifiers for statistical parity (or other related fairness concepts) across a (possibly infinite) collection of groups of bounded VC dimension. However, they do not consider the fact that some groups could be inherently harder to classify than others, and instead work with standard statistical notions such as statistical parity. Doing so results in a non-convex problem -- something that we deal with in our work. 


Our {\BeFair} approach assumes black-box access to an agnostic learning oracle. Such reductions are commonplace in recent work on fairness in machine learning \citep{KearnsNRW18}. For example, \cite{agarwal2018reductions} reduce fair classification to a sequence of cost-sensitive classifications, the solutions of which can be achieved using out-of-the-box classification methods. 

The study of fairness has had a much longer history in economics, in particular, the literature on fair division and cake-cutting \citep{brams1996fair, robertson1998cake}. Out of this line of work have emerged general notions of fairness such as proportionality \citep{steinhaus1948problem}, envy-freeness \citep{varian1973equity}, the core \citep{foley1970lindahl}, and egalitarian (or maxmin) fairness \citep{rawls2009theory, hahne1991round}, to name a few. Of these, the idea of envy-freeness has received great attention in computer science (e.g., \citep{chen2013truth, cohler2011optimal}), and been amenable to adaptation into machine learning as a group-based notion of fairness \citep{balcan2019envy, ustun2019fairness, zafar2017parity}. \citet{HossainMS19} devise algorithms, for multi-class classification that can achieve a variant of group-based (approximate) envy-freeness, with sample complexity of the order of $\log |\mathcal{G}|$, where $\mathcal{G}$ is the collection of pre-defined groups. Our guarantees on the accuracy of the {\PF} classifier are related to the notion of the core: for example, analogous guarantees have been studied in the setting of participatory budgeting \citep{fain2016core}. However, dealing with envy-freeness at an individual level, in the absence of group information, is not very useful. Even in simple binary settings, such as loan and bail applications (all individuals prefer being classified positively, i.e. receiving a loan/bail), satisfying envy-freeness requires all individuals to receive the same outcome \citep{HossainMS19}.

Practical implementations of the above-mentioned economic notions of fairness have garnered interest in the literature on network resource allocation \citep{kleinberg1999fairness, kumar2006fairness}. Proportional fairness \citep{kelly1998rate, bonald2006queueing} has been seen as a way of attaining a middle-ground between welfare maximization and maxmin fairness \citep{jain1984quantitative}. The applicability of proportional fairness to machine learning, in terms of the performance of classification across groups, has, to the best of our knowledge, not been studied before. Although, \cite{li2019fair} do study fair resource allocation based on $\alpha$-fairness (of which proportional fairness is a special case with $\alpha=1$) to improve fairness in terms of the performance across devices in a decentralized federated learning setting.

\subsection{Our model and results}\label{sec:model}
We are given a set of $n$ data points denoted by $\mathcal{N}$, with their features given by $\{x_i\}_{i \in \mathcal{N}}$, and their true binary labels by $\{y_i\}_{i \in \mathcal{N}}$. The hypothesis space at hand will be denoted by $\mathcal{H}$, a set of (deterministic) classifiers. The Boolean variable $u_i(h) \in \{0,1\}$ denotes whether the classifier $h \in \mathcal{H}$ correctly classifies data point $i$. In other words, $u_i(h) = \indicator[h(x_i) = y_i]$. A classification instance is defined by a pair $(\mathcal{N}, \mathcal{H})$. 
We assume that for any classifier $h \in \mathcal{H}$, its \emph{complement} $\bar{h}$, defined by flipping the classification outcomes of $h$ (i.e., $h(x_i) = 1 - \bar{h}(x_i)$), is also in $\mathcal{H}$. This assumption is valid for most natural families of binary classifiers. We denote by $\Delta(\mathcal{H})$ the space of all randomized classifiers over $\mathcal{H}$. If $h \in \Delta(\mathcal{H})$ is obtained via a distribution $D_h$ over $\mathcal{H}$, then for a data point $i \in \mathcal{N}$, we defined the utility $u_i(h) \triangleq \E_{\hpr \sim D_h}[u_i(\hpr)]$.

We are given $\mathcal{G}$, a class of groups, each element of which is of the form $g : \mathcal{N} \to \{1, -1\}$. $g(i) = 1$ means $i$ is in the group and $g(i) = -1$ indicates the opposite. We also use $g$ to denote the subset given by $\{i \in \mathcal{N}: g(i) = 1\}$ and $|g|$ as its size $|\{i \in \mathcal{N}: g(i) = 1\}|$. For any such $g$, its utility under $h$ is $u_g(h) = \frac{1}{|g|} \sum_{i \in g} u_i(h)$. 

For each $g \in \mathcal{G}$, define $h^*_g \triangleq \arg \max_{h \in \mathcal{H}} u_g(h)$ to be the best classifier for the group $g$. The best-effort fairness guarantee is captured via a constraint of the form $f(u_g(h), u_g(\hpr), |g|) \ge 0$. The function $f(\cdot)$ constrains the accuracy $u_g(h)$ of $h$, the classifier at hand, to that of the optimal classifier $h^*_g$ for $g$, with a possible dependence on the size $|g|$ of the group $g$. Applying such a constraint for all $g \in \mathcal{G}$ gives us a uniform best-effort fairness guarantee. What sort of $f(\cdot)$ is workable depends on the class $\mathcal{G}$ considered. 

In Section \ref{sec:pf}, we consider the case where $\mathcal{G}$ includes all the subsets of $\mathcal{N}$, i.e., there is no specific information about $\mathcal{G}$. Via a theoretical worst-case bound (Theorem \ref{thm:lowerbound}), we show that the best we can do in this case is to choose $f(\cdot) = {u}_g(h) - \frac{|g|}{|\mathcal{N}|}[{u}_g(\hpr)]^2$. For any group $g$ that can be {\em perfectly classified} by some $\hpr \in \mathcal{H}$ ($u_g(\hpr) = 1$), the same constraint boils down to $u_g(h) \ge  |g|/|\mathcal{N}|$: in other words, a utility of at least $|g|/|\mathcal{N}|$ should be guaranteed on such a set. Such a guarantee can be interpreted as fairness: If $g$ is a potentially hidden demographic that can be perfectly classified using some features, our classifier should not ignore those features entirely. We show that our {\PF } classifier in fact achieves this guarantee (Theorem \ref{thm:pf}).

In Section \ref{sec:pf_error}, we consider a more streamlined class of groups: in particular, $\mathcal{G}$ contains all linearly separable groups. For ease of exposition, we define the error $\err_g(h) = \sum_{i \in g} [1 - u_i(h)]$, and recast the discussion in terms of it.\footnote{While a fundamentally similar discussion can be done in terms of the utilities, using errors instead leads to an easier handling of the constants involved.} In this case, we seek a much stronger guarantee: $\err_g(h) \le \err_g(h^*_g) + \gamma$. The general form of the optimization problem we solve is as follows:
\begin{align*}
    \min_{h \in \Delta(\hyp)} \quad & \err_\mathcal{N}(h) \\
\text{such that } \forall g \in \gro,  \quad &  \err_g(h^*_g) - \err_g(h) + \gamma \ge 0.
\end{align*}

As discussed in more detail later, to solve the above problem we need to deal with the non-convex constraints. We outline a method ({\BeFair}) to do so when $\mathcal{G}$ consists of linearly separable groups. We will also look for a slightly weaker guarantee as follows: $\err_g(h) \le \delta \cdot \err_g(h^*_g) + \gamma$, for some $\delta \ge 1$ -- weaker because now $\err_g(h)$ has a slightly larger target $\delta u_g(h^*_g)$ to approximate. Our techniques extend seamlessly to such a formulation also.

\section{Best-effort guarantee for all groups}\label{sec:pf}
The first question to ask is whether there is a fundamental limit on how well one can hope to do with respect to fairness in the setting where $\mathcal{G} = 2^\mathcal{N}$. Since we are dealing with a notion of fairness that is measured relative to the family of classifiers at hand, we first want to understand what the best guarantee that can be given is (in the form of a worst-case bound), with no conditions on the type of classifiers used. 

\begin{restatable}{theorem}{lowerbound} \label{thm:lowerbound}
On any data set $\mathcal{N}$, there is no randomized classifier $h$ (for some $\mathcal{H}$) such that for all $g \subseteq \mathcal{N}$ admitting a perfect classifier $h^*_g \in \mathcal{H}$ (i.e., $u_g(h^*_g) = 1$), we have $u_g(h) > \frac{|g|}{|\mathcal{N}|}$.
\end{restatable}

This theorem shows that, in terms of how much utility is accrued by each of the perfectly classified sets, the best bound we can hope to target is one proportional to the fractional size of the given set of data points. Note that for every instance, there exists some $\mathcal{H}$, such that the claim of the theorem holds -- this is not true more generally in the sense that there could exist some $\mathcal{H}$ for which the claim does not hold as shown by Example 2 (in the Appendix).

\subsection{Proportional Fairness ({\PF}) Classifier}
We now demonstrate a classifier that matches the above bound as long as the utilities $u_i(h_j)$'s are binary (which holds in our model, but could also be encountered in other scenarios involving resource allocation discussed in \cref{sec:related}); as mentioned before, this captures multi-class classification as well. The Proportional Fairness classifier is defined as follows:

\begin{definition}[Proportional Fairness ({\PF})]
\label{def:pf}
Given an instance $(\mathcal{N}, \mathcal{H})$, the {\PF} classifier $h_{\PF}$ is the one that maximizes $f(h) \triangleq \sum_{i \in \mathcal{N}} \ln u_i(h)$ over all $h \in \Delta(\mathcal{H})$.
\end{definition}

As mentioned before, the proportional fairness objective has had a long history in network resource allocation literature \citep{kelly1998rate}. However, to the best of our knowledge, its applicability to the classification problem, and the implications thereof, have never been established before.

We now show that the {\PF} classifier achieves a guarantee matching the worst-case bound in Theorem \ref{thm:lowerbound}.

\begin{restatable}{theorem}{pf}
\label{thm:pf}
For any subset $g \subseteq \mathcal{N}$ that admits a perfect classifier $h^*_g \in \mathcal{H}$ (i.e., $u_g(h^*_g) = 1$) we have $u_g(h_{\PF}) \ge \frac{|g|}{|\mathcal{N}|}$.
\end{restatable}

Thus, the {\PF} classifier achieves, on any subset, an accuracy that is \emph{proportional} to the accuracy of the best classifier on that subset scaled by the fractional size of the subset. As mentioned earlier, the use of perfectly classifiable subsets in our analysis is just for the ease of exposition. The results can be suitably translated to using all possible subsets.  For example, the following is a simple corollary of \cref{thm:pf}:
\begin{restatable}{corollary}{pfcor} \label{cor:pf}
For any subset $g \subseteq \mathcal{N}$, with its best classifier $h^*_g = \arg \max_{h \in \mathcal{H}} u_g(h)$, we have $u_g(h_{\PF}) \ge \alpha \left[ u_{g}(h^*_g) \right]^2$, where $\alpha = \frac{|g|}{n}$.
\end{restatable}

Assuming black-box access to an agnostic learning oracle, the {\PF} classifier can be computed using a primal dual style algorithm (details in the Appendix, or see \citet{bhalgat2013optimal} for similar results). In our experiments, we just use a heuristic instead (see Section \ref{sec:experiments}, and also the Appendix). We also do not explicitly discuss the generalization properties -- but we would expect that {\PF} is not prone to overfitting, since all possible groups have to be given a guarantee on performance (details in the Appendix).

\paragraph{Interpreting the results:}
\cref{thm:pf} and \cref{cor:pf} neatly characterize how {\PF} achieves the best possible theoretical bound . One drawback of applying {\PF} in practice is that the theoretical guarantee is proportionally lower for smaller groups, notwithstanding the fact that, in practice, the accuracy of {\PF} on small groups is much better than what is given by these bounds (see \cref{sec:experiments}). As far as the bounds as concerned, the reason that we have to settle for an accuracy proportionally lower for smaller subsets is that the guarantee has to hold for all possible subsets. Some of these subsets could be extremely complex, and possibly unreasonable in most practical settings. As will see next, we can do much better with more restricted classes of groups.
\section{Best-effort guarantees for linearly separable groups}\label{sec:pf_error}

In this section, we limit $\mathcal{G}$ to be a more streamlined class of groups, and aim for a much stronger guarantee. We then devise an algorithm that achieves such a guarantee. As mentioned in Section \ref{sec:model}, we want to find a a randomized classifier $h \in \Delta(\mathcal{H})$ that, for every group $g \in \mathcal{G}$, achieves an absolute error which is within an additive factor $\gamma$ from that of the optimal classifier $h^*_g$ for $g$. In particular, the optimization problem we would like to solve is the following:
\begin{problem}[{\BeFair}$(\gamma)$] For a given hypothesis space $\mathcal{H}$, a class of groups $\mathcal{G}$, and $\gamma \ge 0$,
\begin{align*}
    \min_{h \in \Delta(\hyp)} \quad & \err_{\mathcal{N}}(h) \\
\text{such that } \forall g \in \gro,  \quad &  \err_g(h^*_g) - \err_g(h) + \gamma \ge 0.
\end{align*}
\end{problem}

In particular, we consider $\mathcal{H}$ to be the space of linear hypotheses, and $\mathcal{G}$ the class of all linearly separable groups.\footnote{If $g \in \mathcal{G}$, then $g$ and $\mathcal{N} \setminus g$ are linearly separable.} As mentioned earlier, despite using linear hypotheses and groups, we are faced with a non-convex problem. It can be seen that the non-convexity stems from the best-effort constraint -- while the terms $\err_g(h)$ and $\err_g(h^*_g)$ can be individually made convex by using standard surrogate loss functions, their combination obtained by subtracting one from the other cannot. Note that such a difficulty does not arise for the more absolute notions of fairness such as \emph{parity}, as is the case with the techniques in \citet{KearnsNRW18}. Also, even in our setting, if $\mathcal{G}$ were a small finite set, then all the optimal classifiers $h^*_g$ could be calculated offline, and the corresponding constraints listed to form a simpler convex optimization problem.

We redefine the {\BeFair}($\gamma$) as follows to explicitly factor the hidden optimization problem of finding $h^*(g)$ into the corresponding constraint for $g$; by using the fact that if $\err_g(h^*_g) - \err_g(h) + \gamma \ge 0$, then $\err_g(\hpr) - \err_g(h) + \gamma \ge 0$ for any $\hpr \in {\mathcal{H}}$.

\begin{problem}[{\BeFair}$(\gamma)$]\label{prob:GPF}
\begin{align*}
    \min_{h \in \Delta(\hyp)} \quad & \err_{\mathcal{N}}(h) \\
\text{such that } \forall g \in \gro, \hpr \in \mathcal{H},  \quad &  \err_g(\hpr) - \err_g(h) + \gamma \ge 0.
\end{align*}
\end{problem}



We first define the partial Lagrangian corresponding to Problem \ref{prob:GPF}. Let $\phi(g,h,\hpr) \triangleq - \err_g(\hpr) + \err_g(h) - \gamma$. With dual variables $\lambda_{g, \hpr}$ for every $g \in \mathcal{G}$ and $\hpr \in \mathcal{H}$:
\begin{align*}
    L(h,\lambda) \triangleq \err(h) + \sum_{g \in \gro, \hpr \in \hyp} \lambda_{g,\hpr} \phi(g,h,\hpr).
\end{align*} 

In order to have a convergent algorithm for our optimization, we will restrict the dual space to the bounded set $\Lambda = \{\lambda \in \mathcal{R}_{+}^{ |\mathcal{G} \times \mathcal{H}|} : \Vert\lambda\Vert_1 \le C\}$, where $C$ will be a parameter in our algorithm. Then, by the Minimax Theorem, solving Problem \ref{prob:GPF} is equivalent to solving the following:
\begin{align}\label{eq:minmax}
    \min_{h \in \Delta(\hyp)} \max_{\lambda \in \Lambda}  L(h,\lambda) = \max_{\lambda \in \Lambda} \min_{h \in \Delta(\hyp)}  L(h,\lambda).
\end{align}

The minmax problem can be viewed as a two player zero-sum game: The set of pure strategies for the \emph{learner} (corresponding to the primal) corresponds to $\mathcal{H}$ -- each deterministic classifier $h \in \mathcal{H}$ is a valid pure strategy. For the \emph{adversary} (corresponding to the dual), the pure strategies in $\Lambda$ can be either the all zeros vectors, or a particular choice of $(g,\hpr) \in \mathcal{G}\times\mathcal{H}$. Then, solving Problem \ref{prob:GPF}, via the minmax formulation in Equation \ref{eq:minmax}, is the same as finding an equilibrium of the corresponding two-player zero-sum game with $L(h,\lambda)$ as the payoff for the dual player.

\subsection{Solving the {\BeFair}($\gamma$) problem via a convex relaxation:} \label{sec:solvingbefair}
The equilibrium of a two-player zero-sum game can be found using Fictitious Play, an iterative algorithm which is guaranteed to converge\footnote{The asymptotic convergence is usually fast in practice, and especially so in our experiments.} given that we can solve for the best responses of both players \citep{robinson1951iterative}. Fictitious Play \citep{brown1949some} proceeds in rounds alternating between the primal and dual player: in each round, each player chooses a best response to the the mixed strategy that randomizes uniformly over the empirical history of the other's strategies. A formal description is given in Algorithm \ref{alg:befair}.

\paragraph{Learner's best response:} For a given mixed strategy $\lambda$ of the adversary, the learner needs to solve:
\begin{align*}
    \min_{h \in \Delta(\hyp)}~ \err(h) + \sum_{g \in \gro, \hpr \in \hyp} \lambda_{g,\hpr} \err_g(h).
\end{align*}

Since the optimum is obtained at the corner points of the feasible region of the strategy space, we need only consider pure strategies for the optimization problem above. The learner's problem then becomes:
\begin{align*}
    \min_{h \in \hyp}~ \sum_i w_i \indicator[h(x_i) \neq y_i],
\end{align*}
where $w_i \triangleq 1 + \sum_{g \in \gro, \hpr \in \hyp} \lambda_{g,\hpr} \indicator[g(x_i) = 1]$, and this can be solved since we assume black-box access to a weighted {\ERM} oracle. In practice, many heuristics (like Logistic Regression, Boosting, etc.) are used effectively for this problem, even though it is known to be hard in the worst case \citep{feldman2012agnostic}.



\paragraph{Adversary's best response:} The adversary's best response problem is more involved and will require some novel techniques to solve. Again, we need to optimize only over pure strategies. With a bit of analysis, the dual best response problem can be seen to be equivalent to solving, for a given $h \in \Delta(\mathcal{H})$:
\begin{align}\label{eqn:advprob}
    \min_{g \in \gro, \hpr \in \hyp} \err_g(\hpr) - \err_g(h).
\end{align}
For all $i \in \mathcal{N}$, define $t_i \triangleq \E \indicator[h(x_i) \neq y_i]$. Then the above objective can be written as 
\begin{align}\label{eqn:advobj}
    & \sum_i \indicator[g(x_i) = 1] \left( \indicator[\hpr(x_i) \neq y_i] -  \E \indicator[h(x_i) \neq y_i] \right) \nonumber \\
    = & \sum_i \indicator[g(x_i) \neq -1] \left( \indicator[\hpr(x_i) \neq y_i] - t_i \right),
\end{align}
which is non-convex. For each $i$, we would like to convexify it differently when $t_i = 0$ or $t_i > 0$. (The convex relaxations are not exact.)

Since we only consider the case where both $\mathcal{G}$ and $\mathcal{H}$ consist of linear hypotheses, define $z_{ig} \triangleq x_i^\intercal \theta_g$ and $z_{i\hpr} \triangleq - y x_i^\intercal \theta_{\hpr}$,
where $\theta_g$ and $\theta_{\hpr}$ are the coefficients of the linear hypotheses $g$ and $\hpr$ respectively.

For each $i \in \mathcal{N}$, we add the term $\indicator[z_{ig} > 0] \cdot \indicator[z_{i\hpr} > 0] - t_i\indicator[z_{ig} > 0]$ to the objective. The first term, $\indicator[z_{ig} > 0] \cdot \indicator[z_{i\hpr} > 0]$, is treated as $e^{z_{ig}} \cdot e^{z_{i\hpr}} = e^{z_{ig} + z_{i\hpr}}$; and the second term, $t_i\indicator[z_{ig} > 0]$, is replaced by $t_i(1 - e^{-z_{ig}})$, whence the whole objective becomes $e^{z_{ig} + z_{i\hpr}} + t_i(e^{-z_{ig}} - 1)$, which is convex.

Therefore, we need to solve 
\begin{align}\label{eqn:convexify}
    \min_{\theta_g, \theta_\hpr} \sum_{i \in \mathcal{N}} e^{z_{ig} + z_{i\hpr}} + t_i (e^{-z_{ig}} - 1),
\end{align}
which can be done via convex optimization methods.

\begin{algorithm}[H]
   \caption{Solving {\BeFair($\gamma$)}}
   \label{alg:befair}
\begin{algorithmic}
    \STATE {\bfseries Input:} data set $\mathcal{N}$, $\gamma \ge 0$, number of rounds $T$.
    \STATE Initialize by setting $h_0$ to be some classifier in $\mathcal{H}$, and $\lambda_0$ to be the zero vector.
    \FOR{$t = 1,\ldots T$:}
        \STATE $\bar{h} \gets \mbox{uniform distribution over } \{h_0, \ldots, h_{t-1}\}$
        \STATE $\bar{\lambda} \gets \frac{1}{t}\sum_{t^\prime < t}\lambda_{t^\prime}$
        \STATE $h_t \gets$ Learner's best response to $\bar{\lambda}$
        \STATE $\lambda_t \gets$ Adversary's best response to $\bar{h}$
    \ENDFOR
    
    \STATE {\bfseries Return:} $\bar{\lambda}_t$
\end{algorithmic}
\end{algorithm}

The solution returned by Algorithm \ref{alg:befair} has to be checked for feasibility with respect to Problem \ref{prob:GPF} -- this tells us if the problem is feasible to begin with. Also, the solution to the Adversary's problem (Equation \ref{eqn:advprob}) can potentially be improved by alternately optimizing for $g$ and $\hpr$, \`a la Expectation-Maximization: For a fixed $\hpr$, we can optimize $g$ by using Equation \ref{eqn:advobj} and a convexification analogous to Equation \ref{eqn:convexify}. The converse problem of optimizing $\hpr$, for a fixed $g$, is just a weighted ERM problem.

Note that our technique works as is even for groups that can be defined in terms of any basis expansion of a limited size (a commonly used way of capturing non-linear relationships with linear methods).  For example, if there are two numerical features $x_1,x_2$, then by encoding the values $x_1^2,x_1x_2,x_2^2$ as additional features, we can solve the problem over all groups defined via conic sections (in the feature space).  Similarly, a Boolean AND of binary features can also be written as a linear constraint: for binary features $x_1,x_2$ taking values in $\{0,1\}$, $x_1 \wedge x_2$ is equivalent to $x_1 + x_2 \ge 2$. Extensions to more general classes of $\mathcal{G}$ and $\mathcal{H}$ are an interesting open problem.

\subsection{A more general version of {\BeFair}($\gamma$):} For $\delta \ge 1$, we can generalize Problem \ref{prob:GPF} as follows:
\begin{problem}[{$\delta$-\BeFair}$(\gamma)$]\label{prob:dGPF}
\begin{align*}
    \min_{h \in \Delta(\hyp)} \quad & \err_{\mathcal{N}}(h) \\
\text{such that } \forall g \in \gro,  \quad &  \delta \cdot \err_g(h^*_g) - \err_g(h) + \gamma \ge 0.
\end{align*}
\end{problem}
The only difference from Problem \ref{prob:GPF} is that we have slightly weaker constraints: the error of $h$ on $g$ is compared with $\delta$ times the least possible error on $g$. With a straightforward modification, the overall technique in Section \ref{sec:solvingbefair} works for this problem too. 

For a fixed $\delta$, computing the quantity $\max_{g \in \mathcal{G}}\err_g(h) - \delta \cdot \err_g(h^*_g)$ gives us a way of measuring the fidelity of any given classifier $h$. To do so, the Adversary's problem can be solved (as shown) to find the smallest $\gamma$ for which $h$ becomes feasible for the constraints in Problem \ref{prob:dGPF} (i.e., satisfies the best-effort guarantees). We discuss this in more detail in Section \ref{sec:pfvsbefair}.

\section{Experiments}\label{sec:experiments}


 The primary goal of our experiments is to show that the {\BeFair} algorithm works extremely well in practice. As we discuss below, {\BeFair} achieves its intended purpose (as discussed in the previous section), by achieving strong best-effort fairness guarantees uniformly over all linearly separable groups. In particular, it is able to achieve a performance that is a close approximation of the best possible on these groups (as given in Problem \ref{prob:dGPF}) for small values of $\delta$ and $\gamma$. In addition, we will also evaluate the {\PF} algorithm and show how it behaves differently from {\BeFair}, on account on having to provide guarantees for all possible groups, even those corresponding to a high VC dimension.  We also show how, in practice, the performance of {\PF} seems to be better than what is suggested by the worst-case lower bound via \cref{thm:pf} (which is proportionally weaker for smaller groups).

We work with two data sets: \texttt{adult}, the Adult\footnote{\tiny 48842 instances, 14 features, \url{https://archive.ics.uci.edu/ml/datasets/Adult}} dataset from the UCI Machine Learning Repository, and \texttt{compas}, the COMPAS\footnote{\tiny 6172 instances, 8 features, \url{https://www.propublica.org/article/how-we-analyzed-the-compas-recidivism-algorithm}} Risk of Recidivism data set \citep{angwin2016machine}. Both have binary labels and a mixture of numerical and categorical features. Using these data sets, we compare and contrast the following 
methods (recalling their definitions from earlier):
\begin{enumerate}
    \item {\PF}: As the exact solution of {\PF} is computationally inefficient, we use {\hPF}, a heuristic (details in the Appendix) inspired by Reweighted Approval Voting. \citep{aziz2017justified}. In what follows, we refer to {\hPF} as {\PF}. 
    \item {$\delta$-\BeFair} as described in Section \ref{sec:pf_error}.
    \item {\ERM} methods: We use {\LR} (Logisitic Regression), since it performs best here. We also compare overall accuracy with {\Ada} (AdaBoost), an ensemble method.
    \item The lower bound given by \cref{cor:pf}.
\end{enumerate}

\begin{table}
\small
\setlength\tabcolsep{4pt}
\centering
\begin{tabular}{|c|c|c|c|c|c|}
\hline
    &\LR &\Ada & \hPF & $1.0$-\BeFair & $1.1$-\BeFair\\ \hline
    \texttt{adult} & 0.83 & 0.84 & 0.78 & 0.79 & 0.80 \\ \hline
     \texttt{compas} & 0.75 & 0.75 & 0.64 & 0.70 & 0.71 \\ \hline
\end{tabular}
\caption{Overall test accuracy of various methods.}
\label{tab:overall}
\end{table}

In \cref{tab:overall}, we present the overall accuracy of various methods, i.e., that measured on the entire test set. As there is a trade-off between ensuring fairness for groups and maximizing overall accuracy, {\PF} has a lower overall accuracy compared to other methods. $\delta$-{\BeFair} is much closer to the ERM baselines (especially as seen on the \texttt{compas} dataset, even for a small value of $\delta = 1.1$. Larger values of $\delta$ can only increase accuracy as the fairness constraints become laxer.

\subsection{Evaluating the performance of  {\BeFair}}\label{sec:pfvsbefair}

We first define Maximum Additive Error ($\mae_\delta$), parametrized by $\delta$, of any given classifier $h$.
\begin{definition}[$\mae_\delta(h)$]\label{def:maedelh}
For a given $h$, and $\delta$, 
$\mae_\delta(h) = \max_{g \in \mathcal{G}}\err_g(h) - \delta \cdot \err_g(h^*_g) $.
\end{definition}
For a given $h$, $\mae_\delta(h)$ specifies, for the worst-off group $g$, how much difference there is between the error of $h$ and that of the best classifier for $g$ scaled by $\delta$. 

For instance, $\mae_\delta(\textsc{BeFair})$ can be computed by searching over different values of $\gamma$ to pick the smallest that gives a feasible solution for the $\delta$-\BeFair($\gamma$) problem. On the other hand, $\mae_\delta(\textsc{ERM})$ can be computed by solving the Adversary's problem (Equation \ref{eqn:advobj} modified as per $\delta$) for $h = \ERM$. 

In Figure \ref{fig:alpha_vs_beta}, we compare $\mae_\delta$ of {\ERM} and {\BeFair} for $\delta = 1.0,1.05, \ldots, 1.30$. Errors are reported as a percentage of the entire data set.




\begin{figure}[ht]\centering
\begin{subfigure}{0.44\linewidth}\centering
\includegraphics[scale=0.26]{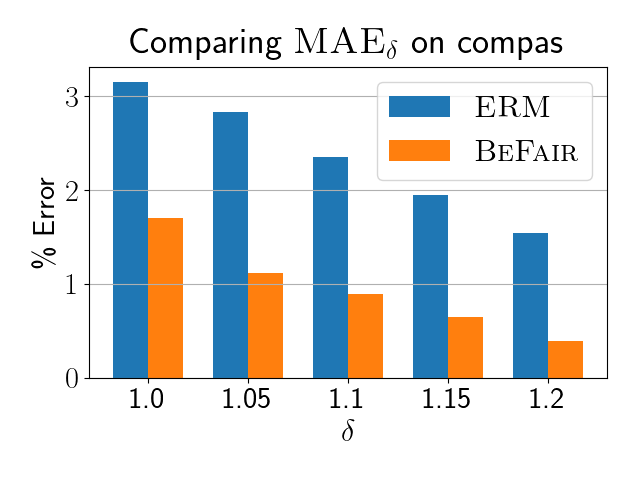}
\end{subfigure}%
\quad
\begin{subfigure}{0.44\linewidth}\centering
\includegraphics[scale=0.26]{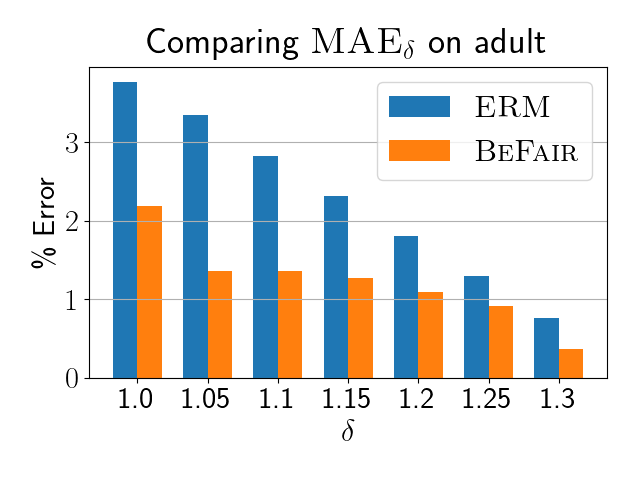}
\end{subfigure}
\caption{Comparing $\mae_\delta$ between {\ERM} and {\BeFair} for varying values of $\delta$.}
\label{fig:alpha_vs_beta}
\end{figure}

Many key observations can be made from this plot: 
\begin{enumerate}[label = (\alph*)]
    \item As we increase $\delta$, the $\mae_\delta$ of both {\ERM} and {\BeFair} decrease. This is because the best-effort constraints get laxer with increasing $\delta$.
    \item Even for $\delta = 1.0$, {\BeFair} achieves an improvement  over {\ERM} of close to 50\% (since $\mae_\delta(\textsc{BeFair})$ is about half of $\mae_\delta(\textsc{ERM})$) on \texttt{compas}, and 33\% on \texttt{adult}, in the $\mae_\delta$ value.
    \item For a slightly larger value of $\delta = 1.10$, we get an extremely low value for $\mae_\delta(\textsc{BeFair})$ of around $1\%$, which means {\BeFair} gets a strong approximation. Therefore, {\BeFair} is able to achieve an multiplicative error of $0.1$, with an additive error of around $1$\%.
    \item $\mae_\delta(\textsc{ERM})$ decreases linearly with $\delta$, while most of the improvement in $\mae_\delta(\textsc{BeFair})$ comes from increasing $\delta$ from $1$ to $1.05$. In other words, {\BeFair} is able to extract a bigger improvement with a small increase of $\delta$.
\end{enumerate}
Overall, {\BeFair} achieves low $\mae_\delta$ for small values of $\delta = 1.05, 1.1$.


\subsection{Comparison of {\PF} with {\BeFair}:}
In \Cref{fig:compas_pf_vs_fp_cum}, we order (on the $x$ axis) the data points in the test set in ascending order of their scores (i.e., confidence of predicting the true label) given by {\LR}. For each point $x$, the $y$ axis shows the accuracy of various methods on the subset consisting of all points from $0$ through $x$.
If we imagine the {\LR} scores as a measure of how easy the points are to classify correctly: then we see that {\PF} gets a more uniform error on all these sets, whereas $1.1$-{\BeFair} has lower error on sets with higher {\LR} scores. This is because the groups where {\BeFair} has high error are probably too complex to be meaningful practically. On the other hand, {\PF} must provide uniform guarantees over all groups, even those corresponding to large VC dimensions, and therefore has a uniformly higher error on them. 
\begin{figure}
    \centering
    \includegraphics[scale=0.15]{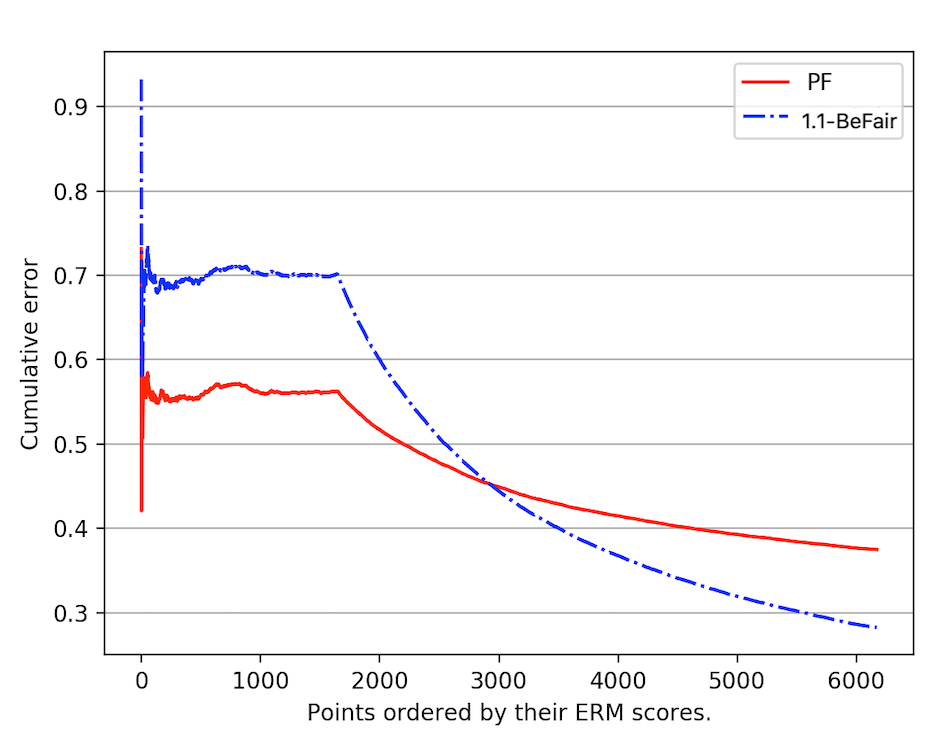}
    \caption{Error on subsets of varying sizes (\texttt{compas}): $x$ axis denotes points in ascending order of {\LR} scores. $y$ axis denotes error accrued on the subset containing points up to $x$.} 
    \label{fig:compas_pf_vs_fp_cum}
\end{figure}

\subsection{Comparison of {\PF} with the theoretical lower bound}\label{subsec:exp:lowerbound}
In \Cref{fig:compas_pf-vs-erm_cum} (left), we order (on the $x$ axis) the data points in the test set in ascending order of their scores (i.e., confidence of predicting the true label) given by {\Ada}. For each point $x$, the $y$ axis shows the accuracy of various methods on the subset consisting of all points from $0$ through $x$. \Cref{fig:compas_pf-vs-erm_cum} (right) does the same with {\PF} scores.

We see that the accuracy of {\PF} is much higher than the worst-case lower bound. {\PF} comes close to the lower bound for larger subsets, especially for those that are easy to classify (see the Appendix for more details). Note that the lower bound \emph{is not monotonic} because it depends on both the size of the subset and the best possible classification accuracy on it (\cref{cor:pf}). Also, in  \Cref{fig:compas_pf-vs-erm_cum} (right), ERM methods do worse because the points with low {\hPF} scores are inherently much harder to classify.

\begin{figure}\centering
\begin{subfigure}{0.44\linewidth}\centering
\includegraphics[scale=0.26]{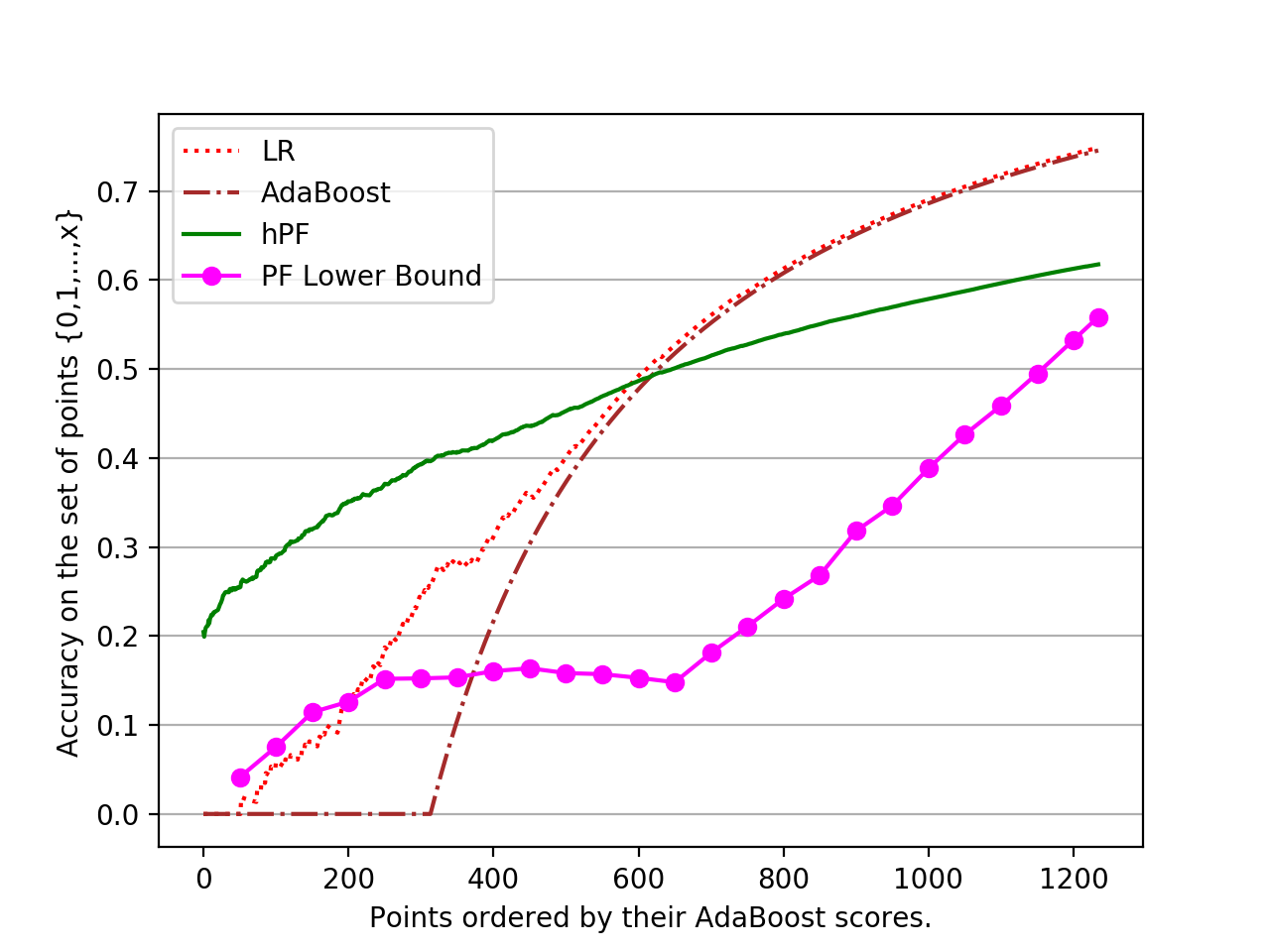}
\end{subfigure}%
\quad
\begin{subfigure}{0.44\linewidth}\centering
\includegraphics[scale=0.26]{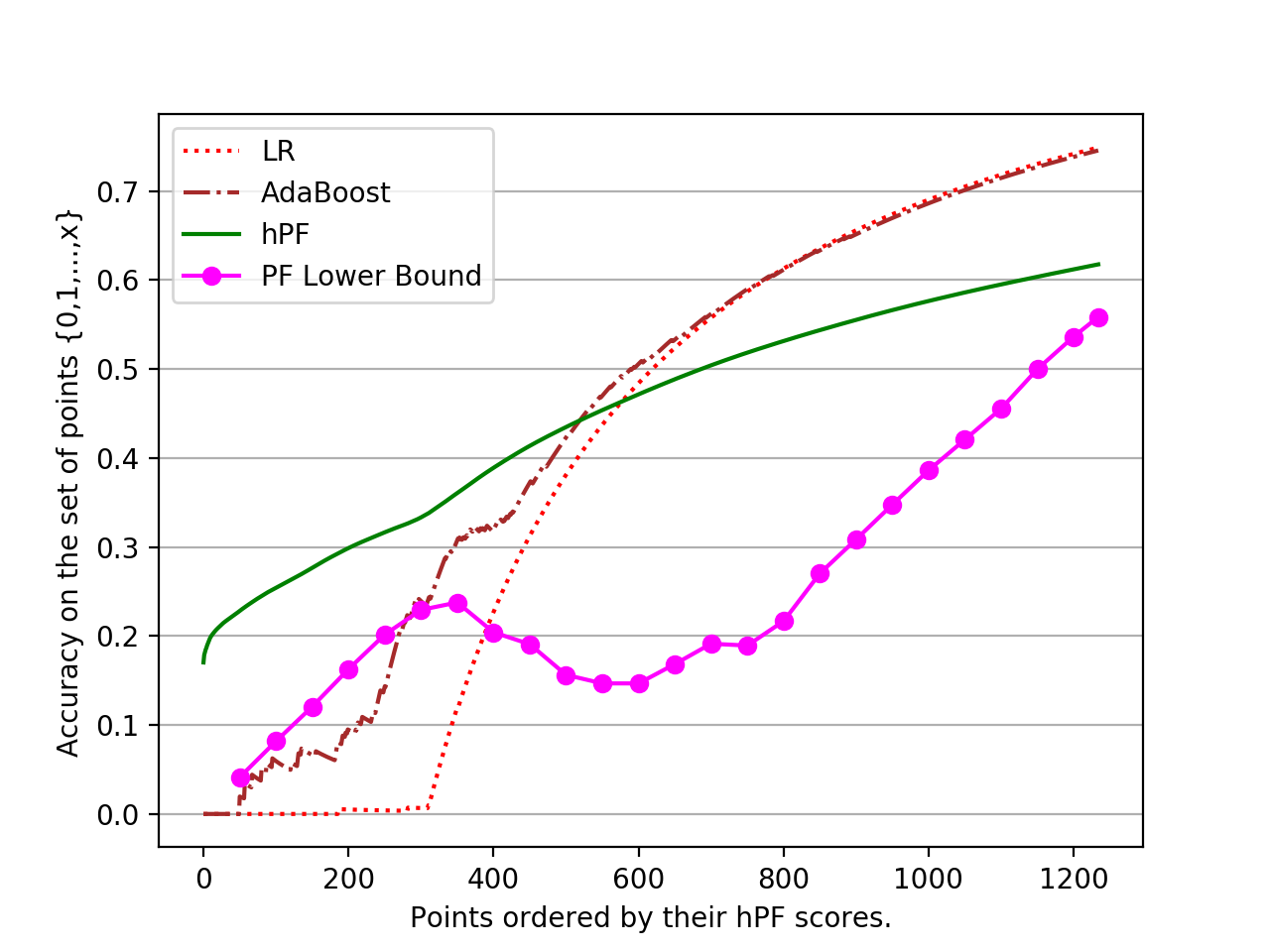}
\end{subfigure}
\caption{Accuracy on subsets of varying size (\texttt{compas}): $x$ axis denotes points in ascending order of {\Ada} (left) and {\hPF} (right) scores. $y$ axis denotes accuracy on the subset containing points up to $x$.}
\label{fig:compas_pf-vs-erm_cum}
\end{figure}

\section{Conclusions}
\label{sec:conclusions}
In this paper, we study group fairness in the (multi-class) classification setting. We propose a notion based on best-effort guarantees, which requires each group in a class $\mathcal{G}$ to have a classification accuracy that is as close as possible to the optimal for that group. When $\mathcal{G}$ consists of all possible groups, we show that {\PF} achieves the theoretical optimum in our setting. When $\mathcal{G}$ consists of linearly separable groups, we can do much better via the {\BeFair} algorithm, which crucially depends on convexification techniques to solve an essentially non-convex problem. We also test our methods on real-world datasets and show that they perform well in practice, especially the {\BeFair} method.


One interesting question for future work is to extend our techniques for more involved classes of groups, say, for example, when $\mathcal{G}$ consists of all groups that can be identified by a fixed neural network. Similar extensions of the hypothesis space $\mathcal{H}$ are also worth looking at. Moreover, in some applications (bail/loan decisions, college admissions, etc.), false negatives and false positives play drastically different roles. Can our framework be extended to deal with such considerations? Can it also be extended to multi-class classification? Note that the guarantees of {\PF} carry over to this setting directly. We would also like to point out that randomized classifiers are not always desirable and have some limitations in practice \citep{cotter2019making}. How to think about best-effort fairness of deterministic classifiers with unknown groups is another interesting open question.

\section*{Acknowledgments}
This work is supported by NSF grant CCF-1637397, ONR award N00014-19-1-2268, and DARPA award FA8650-18-C-7880. Part of this work was done while Yu Cheng was visiting the Institute of Advanced Study.

\bibliography{main.bib}

\begin{thebibliography}{}

\bibitem[Agarwal et~al., 2018]{agarwal2018reductions}
Agarwal, A., Beygelzimer, A., Dud{\'\i}k, M., Langford, J., and Wallach, H.~M.
  (2018).
\newblock A reductions approach to fair classification.
\newblock In {\em Proceedings of the 35th International Conference on Machine
  Learning}, pages 60--69.

\bibitem[Angwin et~al., 2016]{angwin2016machine}
Angwin, J., Larson, J., Mattu, S., and Kirchner, L. (2016).
\newblock Machine bias.
\newblock {\em ProPublica, May}, 23:2016.

\bibitem[Arora et~al., 2012]{arora2012multiplicative}
Arora, S., Hazan, E., and Kale, S. (2012).
\newblock The multiplicative weights update method: a meta-algorithm and
  applications.
\newblock {\em Theory of Computing}, 8(1):121--164.

\bibitem[Aziz et~al., 2017]{aziz2017justified}
Aziz, H., Brill, M., Conitzer, V., Elkind, E., Freeman, R., and Walsh, T.
  (2017).
\newblock Justified representation in approval-based committee voting.
\newblock {\em Social Choice and Welfare}, 48(2):461--485.

\bibitem[Baharlouei et~al., 2019]{baharlouei2019renyi}
Baharlouei, S., Nouiehed, M., Beirami, A., and Razaviyayn, M. (2019).
\newblock R{\'e}nyi fair inference.
\newblock In {\em International Conference on Learning Representations}.

\bibitem[Balcan et~al., 2019]{balcan2019envy}
Balcan, M.-F.~F., Dick, T., Noothigattu, R., and Procaccia, A.~D. (2019).
\newblock Envy-free classification.
\newblock In {\em Advances in Neural Information Processing Systems}, pages
  1238--1248.

\bibitem[Barocas et~al., 2017]{barocas2017fairness}
Barocas, S., Hardt, M., and Narayanan, A. (2017).
\newblock Fairness in machine learning.
\newblock {\em NIPS Tutorial}.

\bibitem[Barocas and Selbst, 2016]{barocas2016big}
Barocas, S. and Selbst, A.~D. (2016).
\newblock Big data's disparate impact.
\newblock {\em Calif. L. Rev.}, 104:671.

\bibitem[Bartlett and Mendelson, 2002]{bartlett2002rademacher}
Bartlett, P.~L. and Mendelson, S. (2002).
\newblock Rademacher and gaussian complexities: Risk bounds and structural
  results.
\newblock {\em Journal of Machine Learning Research}, 3(Nov):463--482.

\bibitem[Berk et~al., 2018]{berk2018fairness}
Berk, R., Heidari, H., Jabbari, S., Kearns, M., and Roth, A. (2018).
\newblock Fairness in criminal justice risk assessments: The state of the art.
\newblock {\em Sociological Methods \& Research}, page 0049124118782533.

\bibitem[Bhalgat et~al., 2013]{bhalgat2013optimal}
Bhalgat, A., Gollapudi, S., and Munagala, K. (2013).
\newblock Optimal auctions via the multiplicative weight method.
\newblock In {\em Proceedings of the fourteenth ACM conference on Electronic
  commerce}, pages 73--90.

\bibitem[Binns, 2018]{binns2017fairness}
Binns, R. (2018).
\newblock Fairness in machine learning: Lessons from political philosophy.
\newblock In {\em Conference on Fairness, Accountability and Transparency},
  pages 149--159.

\bibitem[Bonald et~al., 2006]{bonald2006queueing}
Bonald, T., Massouli{\'e}, L., Proutiere, A., and Virtamo, J. (2006).
\newblock A queueing analysis of max-min fairness, proportional fairness and
  balanced fairness.
\newblock {\em Queueing systems}, 53(1-2):65--84.

\bibitem[Brams and Kilgour, 2014]{brams2014satisfaction}
Brams, S.~J. and Kilgour, D.~M. (2014).
\newblock Satisfaction approval voting.
\newblock In {\em Voting Power and Procedures}, pages 323--346. Springer.

\bibitem[Brams and Taylor, 1996]{brams1996fair}
Brams, S.~J. and Taylor, A.~D. (1996).
\newblock {\em Fair Division: From cake-cutting to dispute resolution}.
\newblock Cambridge University Press.

\bibitem[Brown, 1949]{brown1949some}
Brown, G.~W. (1949).
\newblock {\em Some Notes on Computation of Games Solutions}.
\newblock RAND Corporation, Santa Monica, CA.

\bibitem[Calders et~al., 2009]{calders2009building}
Calders, T., Kamiran, F., and Pechenizkiy, M. (2009).
\newblock Building classifiers with independency constraints.
\newblock In {\em 2009 IEEE International Conference on Data Mining Workshops},
  pages 13--18. IEEE.

\bibitem[Chang, 2011]{chang2011debunking}
Chang, T. (2011).
\newblock Debunking the myth of 'homogeneous' asian students.
\newblock
  \url{https://www.educationworld.com/a_admin/debunking_myth_of_homogeneous_asian_students.shtml}.

\bibitem[Chen et~al., 2019a]{chen2019fairness}
Chen, J., Kallus, N., Mao, X., Svacha, G., and Udell, M. (2019a).
\newblock Fairness under unawareness: Assessing disparity when protected class
  is unobserved.
\newblock In {\em Proceedings of the Conference on Fairness, Accountability,
  and Transparency}, pages 339--348.

\bibitem[Chen et~al., 2019b]{chen2019proportionally}
Chen, X., Fain, B., Lyu, L., and Munagala, K. (2019b).
\newblock Proportionally fair clustering.
\newblock In {\em Proceedings of the 36th International Conference on Machine
  Learning}, pages 1032--1041.

\bibitem[Chen et~al., 2013]{chen2013truth}
Chen, Y., Lai, J.~K., Parkes, D.~C., and Procaccia, A.~D. (2013).
\newblock Truth, justice, and cake cutting.
\newblock {\em Games and Economic Behavior}, 77(1):284--297.

\bibitem[Chouldechova, 2017]{chouldechova2017fair}
Chouldechova, A. (2017).
\newblock Fair prediction with disparate impact: A study of bias in recidivism
  prediction instruments.
\newblock {\em Big data}, 5(2):153--163.

\bibitem[Cohler et~al., 2011]{cohler2011optimal}
Cohler, Y.~J., Lai, J.~K., Parkes, D.~C., and Procaccia, A.~D. (2011).
\newblock Optimal envy-free cake cutting.
\newblock In {\em Twenty-Fifth AAAI Conference on Artificial Intelligence}.

\bibitem[Corbett-Davies et~al., 2017]{corbett2017algorithmic}
Corbett-Davies, S., Pierson, E., Feller, A., Goel, S., and Huq, A. (2017).
\newblock Algorithmic decision making and the cost of fairness.
\newblock In {\em Proceedings of the 23rd ACM SIGKDD International Conference
  on Knowledge Discovery and Data Mining}, pages 797--806.

\bibitem[Cotter et~al., 2019]{cotter2019making}
Cotter, A., Gupta, M., and Narasimhan, H. (2019).
\newblock On making stochastic classifiers deterministic.
\newblock In {\em Advances in Neural Information Processing Systems}, pages
  10910--10920.

\bibitem[Dwork et~al., 2012]{dwork2012fairness}
Dwork, C., Hardt, M., Pitassi, T., Reingold, O., and Zemel, R. (2012).
\newblock Fairness through awareness.
\newblock In {\em Proceedings of the 3rd innovations in theoretical computer
  science conference}, pages 214--226.

\bibitem[Dwork et~al., 2018]{dwork2018decoupled}
Dwork, C., Immorlica, N., Kalai, A.~T., and Leiserson, M. (2018).
\newblock Decoupled classifiers for group-fair and efficient machine learning.
\newblock In {\em Conference on Fairness, Accountability and Transparency},
  pages 119--133.

\bibitem[Fain et~al., 2016]{fain2016core}
Fain, B., Goel, A., and Munagala, K. (2016).
\newblock The core of the participatory budgeting problem.
\newblock In {\em International Conference on Web and Internet Economics},
  pages 384--399. Springer.

\bibitem[Feldman et~al., 2012]{feldman2012agnostic}
Feldman, V., Guruswami, V., Raghavendra, P., and Wu, Y. (2012).
\newblock Agnostic learning of monomials by halfspaces is hard.
\newblock {\em SIAM Journal on Computing}, 41(6):1558--1590.

\bibitem[Foley, 1970]{foley1970lindahl}
Foley, D.~K. (1970).
\newblock Lindahl's solution and the core of an economy with public goods.
\newblock {\em Econometrica: Journal of the Econometric Society}, pages 66--72.

\bibitem[Friedler et~al., 2016]{friedler2016possibility}
Friedler, S.~A., Scheidegger, C., and Venkatasubramanian, S. (2016).
\newblock On the (im)possibility of fairness.
\newblock {\em arXiv preprint arXiv:1609.07236}.

\bibitem[Gr{\"{o}}tschel et~al., 1988]{GLS1988}
Gr{\"{o}}tschel, M., Lov{\'{a}}sz, L., and Schrijver, A. (1988).
\newblock {\em Geometric Algorithms and Combinatorial Optimization}, volume~2
  of {\em Algorithms and Combinatorics}.
\newblock Springer.
\newblock Second Edition, 1993.

\bibitem[Hahne, 1991]{hahne1991round}
Hahne, E.~L. (1991).
\newblock Round-robin scheduling for max-min fairness in data networks.
\newblock {\em IEEE Journal on Selected Areas in communications},
  9(7):1024--1039.

\bibitem[Hajian and Domingo-Ferrer, 2012]{hajian2012methodology}
Hajian, S. and Domingo-Ferrer, J. (2012).
\newblock A methodology for direct and indirect discrimination prevention in
  data mining.
\newblock {\em IEEE transactions on knowledge and data engineering},
  25(7):1445--1459.

\bibitem[Hardt et~al., 2016]{hardt2016equality}
Hardt, M., Price, E., and Srebro, N. (2016).
\newblock Equality of opportunity in supervised learning.
\newblock In {\em Advances in neural information processing systems}, pages
  3315--3323.

\bibitem[Hashimoto et~al., 2018]{hashimoto2018fairness}
Hashimoto, T.~B., Srivastava, M., Namkoong, H., and Liang, P. (2018).
\newblock Fairness without demographics in repeated loss minimization.
\newblock In {\em Proceedings of the 35th International Conference on Machine
  Learning}, pages 1934--1943.

\bibitem[Hoggett et~al., 2013]{hoggett2013fairness}
Hoggett, P., Wilkinson, H., and Beedell, P. (2013).
\newblock Fairness and the politics of resentment.
\newblock {\em Journal of Social Policy}, 42(3):567--585.

\bibitem[Hossain et~al., 2020]{HossainMS19}
Hossain, S., Mladenovic, A., and Shah, N. (2020).
\newblock Designing fairly fair classifiers via economic fairness notions.
\newblock In {\em Proceedings of the 29th International World Wide Web
  Conference}.

\bibitem[Jain et~al., 1984]{jain1984quantitative}
Jain, R.~K., Chiu, D.-M.~W., and Hawe, W.~R. (1984).
\newblock A quantitative measure of fairness and discrimination.
\newblock {\em Eastern Research Laboratory, Digital Equipment Corporation,
  Hudson, MA}.

\bibitem[Kamiran and Calders, 2012]{kamiran2012data}
Kamiran, F. and Calders, T. (2012).
\newblock Data preprocessing techniques for classification without
  discrimination.
\newblock {\em Knowledge and Information Systems}, 33(1):1--33.

\bibitem[Kearns et~al., 2018]{KearnsNRW18}
Kearns, M.~J., Neel, S., Roth, A., and Wu, Z.~S. (2018).
\newblock Preventing fairness gerrymandering: Auditing and learning for
  subgroup fairness.
\newblock In {\em Proceedings of the 35th International Conference on Machine
  Learning}, pages 2569--2577.

\bibitem[Kelly et~al., 1998]{kelly1998rate}
Kelly, F.~P., Maulloo, A.~K., and Tan, D.~K. (1998).
\newblock Rate control for communication networks: shadow prices, proportional
  fairness and stability.
\newblock {\em Journal of the Operational Research society}, 49(3):237--252.

\bibitem[Khachiyan, 1979]{khachiyan1979polynomial}
Khachiyan, L.~G. (1979).
\newblock A polynomial algorithm in linear programming.
\newblock In {\em Doklady Akademii Nauk}, volume 244, pages 1093--1096. Russian
  Academy of Sciences.

\bibitem[Kim et~al., 2019]{kim2019multiaccuracy}
Kim, M.~P., Ghorbani, A., and Zou, J. (2019).
\newblock Multiaccuracy: Black-box post-processing for fairness in
  classification.
\newblock In {\em Proceedings of the 2019 AAAI/ACM Conference on AI, Ethics,
  and Society}, pages 247--254.

\bibitem[Kleinberg, 2018]{kleinberg2018inherent}
Kleinberg, J. (2018).
\newblock Inherent trade-offs in algorithmic fairness.
\newblock In {\em Abstracts of the 2018 ACM International Conference on
  Measurement and Modeling of Computer Systems}, pages 40--40.

\bibitem[Kleinberg et~al., 1999]{kleinberg1999fairness}
Kleinberg, J., Rabani, Y., and Tardos, {\'E}. (1999).
\newblock Fairness in routing and load balancing.
\newblock In {\em 40th Annual Symposium on Foundations of Computer Science},
  pages 568--578. IEEE.

\bibitem[Koren, 2016]{koren2016does}
Koren, J.~R. (2016).
\newblock What does that web search say about your credit.
\newblock {\em Los Angeles Times}.

\bibitem[Kumar and Kleinberg, 2006]{kumar2006fairness}
Kumar, A. and Kleinberg, J. (2006).
\newblock Fairness measures for resource allocation.
\newblock {\em SIAM Journal on Computing}, 36(3):657--680.

\bibitem[Li et~al., 2019]{li2019fair}
Li, T., Sanjabi, M., Beirami, A., and Smith, V. (2019).
\newblock Fair resource allocation in federated learning.
\newblock In {\em International Conference on Learning Representations}.

\bibitem[Madras et~al., 2018]{madras2018learning}
Madras, D., Creager, E., Pitassi, T., and Zemel, R. (2018).
\newblock Learning adversarially fair and transferable representations.
\newblock In {\em International Conference on Machine Learning}, pages
  3384--3393. PMLR.

\bibitem[Meier and Melton, 2012]{meier2012latino}
Meier, K.~J. and Melton, E.~K. (2012).
\newblock Latino heterogeneity and the politics of education: The role of
  context.
\newblock {\em Social science quarterly}, 93(3):732--749.

\bibitem[Papadimitriou and Roughgarden, 2008]{papadimitriou2008computing}
Papadimitriou, C.~H. and Roughgarden, T. (2008).
\newblock Computing correlated equilibria in multi-player games.
\newblock {\em Journal of the ACM (JACM)}, 55(3):1--29.

\bibitem[Pedregosa et~al., 2011]{scikit-learn}
Pedregosa, F., Varoquaux, G., Gramfort, A., Michel, V., Thirion, B., Grisel,
  O., Blondel, M., Prettenhofer, P., Weiss, R., Dubourg, V., Vanderplas, J.,
  Passos, A., Cournapeau, D., Brucher, M., Perrot, M., and Duchesnay, E.
  (2011).
\newblock Scikit-learn: Machine learning in {P}ython.
\newblock {\em Journal of Machine Learning Research}, 12:2825--2830.

\bibitem[Rajkomar et~al., 2018]{rajkomar2018ensuring}
Rajkomar, A., Hardt, M., Howell, M.~D., Corrado, G., and Chin, M.~H. (2018).
\newblock Ensuring fairness in machine learning to advance health equity.
\newblock {\em Annals of internal medicine}, 169(12):866--872.

\bibitem[Rawls, 2009]{rawls2009theory}
Rawls, J. (2009).
\newblock {\em A theory of justice}.
\newblock Harvard University Press.

\bibitem[Rezaei et~al., 2020]{rezaei2020fairness}
Rezaei, A., Fathony, R., Memarrast, O., and Ziebart, B. (2020).
\newblock Fairness for robust log loss classification.
\newblock In {\em Proceedings of the AAAI Conference on Artificial
  Intelligence}, volume~34, pages 5511--5518.

\bibitem[Robertson and Webb, 1998]{robertson1998cake}
Robertson, J. and Webb, W. (1998).
\newblock {\em Cake-cutting algorithms: Be fair if you can}.
\newblock AK Peters/CRC Press.

\bibitem[Robinson, 1951]{robinson1951iterative}
Robinson, J. (1951).
\newblock An iterative method of solving a game.
\newblock {\em Annals of Mathematics}, 54(2):296--301.

\bibitem[Shalev-Shwartz and Ben-David, 2014]{shalev2014understanding}
Shalev-Shwartz, S. and Ben-David, S. (2014).
\newblock {\em Understanding machine learning: From theory to algorithms}.
\newblock Cambridge university press.

\bibitem[Steinhaus, 1948]{steinhaus1948problem}
Steinhaus, H. (1948).
\newblock The problem of fair division.
\newblock {\em Econometrica}.

\bibitem[Ustun et~al., 2019]{ustun2019fairness}
Ustun, B., Liu, Y., and Parkes, D. (2019).
\newblock Fairness without harm: Decoupled classifiers with preference
  guarantees.
\newblock In {\em International Conference on Machine Learning}, pages
  6373--6382.

\bibitem[Varian, 1974]{varian1973equity}
Varian, H.~R. (1974).
\newblock Equity, envy, and efficiency.
\newblock {\em Journal of Economic Theory}, 9(1):63--91.

\bibitem[Zafar et~al., 2017]{zafar2017parity}
Zafar, M.~B., Valera, I., Rodriguez, M., Gummadi, K., and Weller, A. (2017).
\newblock From parity to preference-based notions of fairness in
  classification.
\newblock In {\em Advances in Neural Information Processing Systems}, pages
  229--239.

\end{thebibliography}

\newcommand{\NN}{\mathcal{N}}
\newcommand{\HH}{\mathcal{H}}
\newcommand{\poly}{\mathrm{poly}}
\newcommand{\OPT}{\mathrm{OPT}}

\appendix

\section{Omitted proofs}\label{app:proofs}
We first recall Theorem \ref{thm:lowerbound}:
\lowerbound*


\begin{proof}
Suppose there are only two classifiers $h_1, h_2$ in $\mathcal{H}$: $h_1$ classifies set $g_1$ correctly, and $h_2$ (equivalently $\overline{h_1}$) classifies $g_2 = \mathcal{N} \setminus g_1$ correctly. If an algorithm chooses $h_1$ with probability $p \geq 0$ and $h_2$ with probability $1-p$, then the average utility on $g_1$ is $p$, and that on $g_2$ is $1-p$. Clearly, we cannot simultaneously have $p > \frac{|g_1|}{n}$ and $1 - p > 1 - \frac{|g_1|}{n}$.
\end{proof}

Next we look at Theorem \ref{thm:pf}:
\pf*
\begin{proof}
Let $h_\mathsf{\PF}$ be the {\PF} classifier. Thus $h = h_{\PF}$ maximizes $f(h) := \sum_{i \in \mathcal{N}} \ln u_i(h)$. Therefore, for any $h \in \mathcal{H}$ and any $\varepsilon \geq 0$,
\begin{align*}
    f(\varepsilon \cdot h + (1 - \varepsilon) \cdot h_{\PF}) - f(h_{\PF}) 
    = \sum_{i \in \mathcal{N}} \ln u_i(\varepsilon \cdot h + (1 - \varepsilon) \cdot h_{\PF}) - \ln u_i(h_{\PF}) \leq 0.
\end{align*}
Since the above expression attains its maxima at $\varepsilon = 0$, we take the derivative with respect to $\varepsilon$, and evaluate it at $\varepsilon = 0$, to get:
\begin{align*}
    \sum_{i \in \mathcal{N}} (u_i(h) - u_i(h_{\PF})) \cdot \frac{1}{u_i(h_{\PF})} \leq 0.
\end{align*}

Define $n = |\mathcal{N}|$. Rearranging the inequality above, we get $\sum_{i \in \mathcal{N}} \frac{u_i(h)}{u_i(h_{\PF})} \leq n$, and consequently,
\begin{align*}
    \sum_{i \in \mathcal{N}: u_i(h) = 1} \frac{1}{u_i(h_{\PF})} \leq \sum_{i \in \mathcal{N}} \frac{u_i(h)}{u_i(h_{\PF})} \leq n.
\end{align*}

If $g \subseteq \mathcal{N}$ of size $\alpha n$ is perfectly classified by $h$, then the above inequality gives
\begin{align*}
    \sum_{i \in g} \frac{1}{u_i(h_{\PF})} \leq n,
\end{align*}
which in turn implies that
\begin{align*}
    \frac{|g|}{\sum_{i \in g} \frac{1}{u_i(h_{\PF})}} \geq  \frac{|g|}{n} = \alpha.
\end{align*}
Since the arithmetic mean is at least the harmonic mean, we get
\[
    u_{g}(h_{\PF}) =  \frac{1}{|g|} \sum_{i \in g} u_i(h_{\PF}) \ge \alpha. \qedhere
\]
\end{proof}

Recall \cref{cor:pf}:
\pfcor*
\begin{proof}
For any $T \subseteq g$, we have:
\begin{align*}
    u_g(h_{\PF}) \ge u_T(h_{\PF}) \frac{|T|}{|g|}.
\end{align*}
Let $T$ be the largest subset of $g$ that is perfectly classifiable by $h_j$. By \cref{thm:pf}, $u_T(h_{\PF}) \ge \alpha \frac{|T|}{|g|}$. Then, we get
\begin{align*}
    u_g(h_{\PF}) \ge u_T(h_{\PF}) \frac{|T|}{|g|} \ge \alpha \left(\frac{|T|}{|g|}\right)^2.
\end{align*}
Since $\frac{|T|}{|g|} = u_g(h^*_g)$, the above inequality turns into
\[
    u_g(h_{\PF}) \ge \alpha \left[u_g(h^*_g)\right]^2. \qedhere
\]
\end{proof}

\section{Omitted examples}\label{app:examples}
\setcounter{example}{1}
In the following example, we see that there are instances where, for some specific $\mathcal{H}$, the claim of \cref{thm:lowerbound} does not hold. 
\begin{example}\label{ex:lowerbound}
Suppose there are three data points $\{1,2,3\}$, and there are four classifers $h_1, \bar{h_1}, h_2 \mbox{ and } \bar{h_2}$ in $\mathcal{H}$. Classifier $h_1$ classifies $\{1,2\}$ correctly, and $h_2$ classifies $\{2,3\}$ correctly. If a randomized classifier $h$ picks $h_1$ and $h_2$ with the same probability $1/2$, then for all subsets $S$ which are perfectly classifiable, i.e., for each of $\{1\}$, $\{2\}$, $\{3\}$, $\{1,2\}$ and $\{2,3\}$, the utility $u_S(h)$ is $0.5$, $1$, $0.5$, $0.75$ and $0.75$, respectively. Each of these utilities is greater than the fractional size of the subsets, which does not agree with the claim of \cref{thm:lowerbound}. 
\end{example}

\section{Computing the {\sc \texttt{Proportional Fairness}} classifier}\label{app:computepf}
In this section, we describe how to compute our {\sc \texttt{Proportional Fairness}} ({\PF}) classifier.
We present computational results for two different settings.
Lemma~\ref{lem:pf-comp-finite} states that when the set of deterministic classifiers are given explicitly, we can compute the {\PF} classifiers in polynomial time.
Lemma~\ref{lem:pf-comp-inf} focuses on the case where there are exponentially or infinitely many deterministic classifiers, and shows that the {\PF} classifier can still be computed in polynomial time, assuming that we have black-box access to an agnostic learning oracle.

Recall that $\NN$ is the set of data-points with $|\NN| = n$, $\mathcal{H}$ is the set of (deterministic) classifiers with $|\HH| = m$, and $u_i(h) = 1$ if the classifier $h$ labels the $i$-th data-point correctly, and $u_i(h) = 0$ otherwise.

The {\PF} classifier is a distribution $(p_j)_{j: h_j \in \mathcal{H}}$ over deterministic classifiers. {\PF} corresponds to the optimal solution to the following mathematical program:
\begin{equation}
\begin{array}{ll}
\text{maximize}   & \sum_{i \in \NN} \ln v_i \\
\text{subject to} & v_i \le \sum_{{h_j} \in \HH} p_j u_i(h_j), \quad \forall i \in \NN \\
                  & \sum_{h_j \in \HH} p_j \le 1\\
                  & p_j \ge 0, \quad \forall h_j \in \HH
\end{array}
\label{eqn:pf}
\end{equation}
where the variables $p_j$ describes the randomized classifier (which chooses $h_j$ with probability $p_j$); and $0 \le v_i \le 1$ is the utility of the $i$-th data-point under the distribution $p$ (i.e., the probability that the $i$-th data-point is classified correctly).
Observe that the objective function is monotone in every $v_i$, so at optimality, we always have $\sum_{h_j \in \HH} p_j = 1$, and $v_i = \sum_{{h_j} \in \HH} p_j u_i(h_j)$ for every $i \in \NN$.

\begin{lemma}
\label{lem:pf-comp-finite}
Given a classification instance $(\NN, \HH)$ with $n = |\NN|$ data-points and $m = |\HH|$ classifiers, the mathematical program~\eqref{eqn:pf} can be solved to precision $\eps > 0$ in time $\poly(n, m, \log(1/\eps))$.
\end{lemma}
\begin{proof}
When $|\NN| = n$ and $|\HH| = m$, the mathematical program~\eqref{eqn:pf} has $n + m$ variables.
The feasible region is given explicitly by a set of $n+1$ linear constraints.
Because the objective function is concave in the variables $(v, p)$, we can minimize it using the ellipsoid method~\citep{khachiyan1979polynomial,GLS1988} in time $\poly(n, m, \log(1/\eps))$.
\end{proof}

In many applications, we often have infinitely many hypotheses in $\HH$ (e.g., all hyperplanes in $\mathbb{R}^d$).
If this is the case, the mathematical program~\eqref{eqn:pf} has infinitely many variables, and to solve it, we need to make some assumptions on the structure of $\HH$.

A commonly used assumption is that there exists an agnostic learning oracle: given a set of weights on the data-points, the oracle returns an optimal classifier $h \in \HH$ subject to these weights.
Formally, we assume black-box access to an oracle for the following problem:

\begin{definition}[Agnostic Learning]
Fix a set of data-points $\NN$ and a family of classifiers $\HH$.
Given any weights $(w_i)_{i \in \NN}$, find a classifier $h \in \HH$ that maximizes the (weighted) average accuracy on $\NN$.
That is, $h$ maximizes $\sum_{i \in \NN} w_i u_i(h)$.
\end{definition}

Given such an oracle, there are a few ways in which Problem~\eqref{eqn:pf} can be solved theoretically in polynomial time. We first sketch the outline of a \emph{multiplicative weights} approach: Consider the feasibility version of the problem, where we want to check if there is a feasible solution that gives us an objective value of at least $U^*$. If we define a region $P$ as the set of $v, p$, where $v = \{v_i\}_{i \in \mathcal{N}}$ and $p = \{p_j\}_{j: h_j \in \mathcal{H}}$, that satisfy: 
\begin{align}
    \sum_{i \in \mathcal{N}} \ln v_i &\ge U^*, \mbox{ and }  v_i \ge 0 \quad \forall i \in \mathcal{N}, \\
    \sum_{j: h_j \in \mathcal{H}} p_j &= 1, \mbox{ and } p_j \ge 0 \quad \forall j:h_j \in \mathcal{H},
\end{align}
then the problem can be restated as:
\begin{align*}
    \exists ? (v,p) \in P \mbox{ such that } \sum_{j: h_j \in \mathcal{H}}p_j u_i(h_j) \ge v_i, \quad \forall i \in \mathcal{N}.
\end{align*}
This can be solved via the multiplicative weights if we have an efficient oracle for solving the following optimization problem for a given $y \ge 0$ \citep{bhalgat2013optimal, arora2012multiplicative}:
\begin{align*}
    \mbox{maximize} \quad  &\sum_{i \in \mathcal{N}} y_i\left( \sum_{j: h_j \in \mathcal{H}}p_j u_i(h_j) - v_i \right) \\
    \mbox{subject to} \quad & (v,p) \in P.
\end{align*}
It can be seen that the above devolves into two decoupled problems: The first one is for $v$, which can be solved analytically,
\begin{align*}
    \mbox{minimize} &\sum_{i \in \mathcal{N}} y_i v_i \\
    \mbox{subject to } & \sum_{i \in \mathcal{N}} \ln(v_i) \ge U^*\\
    & v \ge 0,
\end{align*}
and the second for $p$, which can be solved with access to an agnostic learning oracle:
\begin{align*}
    \mbox{minimize} &\sum_{i \in \mathcal{N}} \sum_{j : h_j \in \mathcal{H}} p_j y_i u_i(h_j) \\
    \mbox{subject to } & \sum_{j: h_j \in \mathcal{H}} p_j = 1\\
    & p \ge 0.
\end{align*}

As the next lemma shows, Problem \eqref{eqn:pf} can also be theoretically solved in a more straightforward way using the ellipsoid method.
\begin{lemma}
\label{lem:pf-comp-inf}
Given a classification instance $(\NN, \HH)$ with $n = |\NN|$ data-points and an agnostic learning oracle for the set of classifiers $\HH$, the mathematical program~\eqref{eqn:pf} can be solved to precision $\eps > 0$ in time $\poly(n, \log(1/\eps))$.
\end{lemma}
\begin{proof}
We use the ``Ellipsoid Against Hope'' algorithm proposed in~\citet{papadimitriou2008computing}.

Consider the dual of~\eqref{eqn:pf}:
\begin{equation}
\begin{array}{ll}
\text{minimize}   & -\left(\sum_{i \in \NN} \ln w_i\right) - n + z \\
\text{subject to} & \sum_i w_i u_i(h_j) \le z, \quad \forall h_j \in \HH \\
                  & w_i \ge 0, \quad \forall i \in \NN. 
\end{array}
\label{eqn:pf-dual}
\end{equation}
The dual program~\eqref{eqn:pf-dual} has $n+1$ variables $(w, z)$ and infinitely many constraints.
Strong duality holds despite the infinite-dimensionality of~\eqref{eqn:pf}.

Let $\OPT$ denote the optimal value of the primal and dual programs.
The dual objective is convex in $(w, z)$, so we can add a constraint
\begin{equation}
-\left(\sum_{i \in \NN} \ln w_i\right) - n + z \le \OPT - \eps
\label{eqn:ellipsoid-against}
\end{equation}
and use ellipsoid method to check the feasibility of the dual with this additional constraint.
We show a separation oracle exists so we can run the ellipsoid method.
For a fixed point $(w, z)$, we can verify Constraint~\eqref{eqn:ellipsoid-against} directly, and we can use the agnostic learning oracle to check the infinitely many constraints
\begin{equation}
\sum_i w_i u_i(h_j) \le z, \quad \forall h_j \in \HH
\label{eqn:pf-dual-hj-con}
\end{equation}
because it is sufficient to first use the oracle to find some $h_j$ that maximizes $\sum_i w_i u_i(h_j)$, and then check only the $j$-th constraint.

We know the ellipsoid method must conclude infeasibility, because the minimum possible value of the dual is $\OPT$ but we are asking for $\OPT - \eps$.
Let $\HH'$ denote the set of classifiers whose corresponding constraints~\eqref{eqn:pf-dual-hj-con} are checked in the execution of the ellipsoid method.
Because the ellipsoid only examines the constraints in $\HH'$ and concludes infeasibility, we know that if we replace $\HH$ with $\HH'$ in the dual program~\eqref{eqn:pf-dual}, the objective value is still larger than $\OPT - \eps$.
Moreover, the cardinality of $\HH'$ is at most $\poly(n, \log(1/\eps))$ because the ellipsoid method terminates in $\poly(n, \log(1/\eps))$ steps.

Consequently, we can replace $\HH$ with $\HH'$ in the primal convex program~\eqref{eqn:pf} to make it finite-dimensional, and invoke Lemma~\ref{lem:pf-comp-finite} to solve it where $m = |\HH'| = \poly(n, \log(1/\eps))$.
Therefore, the overall running time is $\poly(n, \log(1/\eps))$.
\end{proof}

It is worth noting that the proof of Lemma~\ref{lem:pf-comp-inf} continues to hold even if we only have an $\eps$-approximately optimal agnostic learning oracle.

\section{A heuristic for {\PF}} \label{app:pf_heuristic}
Here we describe a heuristic for {\PF} drawing inspiration from some literature on social choice (voting). Imagine a social choice setting: we have a set $N$ of voters and a set $C$ of candidates. Each voter $i \in N$ has a subset $A_i \subseteq C$ of candidates which she prefers. The goal here is to select a committee of a fixed size $k$, i.e., a subset $W \subseteq C$ ($|W| = k$), which ``satisfies'' the voters as much as possible 

One method to do so is Proportional Approval Voting \citep{aziz2017justified}, or PAV for short. Here, a voter is assumed (for the sake of computation) to derive a utility of $1 + \frac12 + \ldots + \frac1j$ from a committee $W$ that contains exactly $j$ of her
approved candidates, i.e., $|A_i \cap W| = j$. For PAV, as a method of computing a committee $W$, the overall goal is to maximize the sum of the voters’ utilities -- in other words, PAV outputs a set $W^* = \arg \max_{W \subseteq C:|W|=k} \sum_{i \in N} |A_i \cap W|$.

Reweighted Approval Voting (RAV for short) converts PAV into a multi-round rule as follows: Start by setting $W = \varnothing$. Then in round $j$ ($j = 1, \ldots, k$), select (without replacement) a candidate $c$ which maximizes $\sum_{i \in N: c \in A_i} \frac{1}{1 + |W \cap A_i|}$, and adds it to $W$. Finally, it outputs the set $W$, after $k$ rounds. i.e., having chosen $k$ candidates. RAV is also sometimes referred to as ``Sequential PAV” \citep{brams2014satisfaction}. 

In our case, for {\PF}, we want to maximize $f(h) := \sum_{i \in \mathcal{N}} \ln u_i(h)$. Since $\ln(t) \approx \sum_{i=1}^t \frac1i$, we can think of PAV as a close enough proxy (where the voters are the data points, and the candidates are all the available classifiers). In the same token, we can potentially apply RAV to our problem (assuming black box access to an agnostic learning oracle). In fact, in practice, we find that a slight modification to RAV works better. We describe this in terms of our problem setting and notation in \cref{alg:pf_heuristic}. 

\begin{algorithm}[ht!]
   \caption{{\PF} Heuristic Classifier}
   \label{alg:pf_heuristic}
\begin{algorithmic}
    \STATE {\bfseries Input:} data set $\mathcal{N}$, family of classifiers  $\mathcal{H}$, number of iterations $R$
    \STATE Initialize: $r \gets 1$, $(c_i)_{i \in \mathcal{N}} \gets \mathbf{0}$, $(w_i)_{i \in \{1,\ldots,R\}} \gets \mathbf{0}$
    \WHILE{$r \leq R$}
        \STATE $h^*_r \gets \mbox{argmax}_{h_j \in \mathcal{H}} \sum_{i \in \mathcal{N}} \frac{1}{1 + c_i} \cdot u_i(h_j)$
        \STATE $\mathcal{T}_r \gets \{i \in \mathcal{N} \mid u_i(h^*_r) = 1\}$
        \STATE $w_r \gets \sum_{i \in \mathcal{T}_r} \frac{1}{1 + c_i}$
        \FORALL{$i \in \mathcal{T}_r$}
            \STATE $c_i \gets c_i + 1$
        \ENDFOR
        \STATE $r \gets r+1$
    \ENDWHILE
    
    \STATE {\bfseries Return:} classifier $h_{\PF}$ that chooses $h_t^* \in \mathcal{H}$ with probability $p_t \propto w_t$ for $t = 1,2,\ldots,R$.
\end{algorithmic}
\end{algorithm}

Note that a direct analog of RAV would be choosing $h^*_t$ with probability just $\frac1R$, i.e., uniform over all $r \in \{1,\ldots,R\}$. Also note that any implementation of RAV would not have black-box access to an agnostic learning oracle. And when data points are reweighted each time, standard out-of-the-box training methods apply the weights to the respective gradient updates. Since scaling the gradient updates affects the convergence of the training procedure, some kind of rescaling is necessary. In practice, we find that reweighting the probabilities to be proportional to $w_t$ (see \cref{alg:pf_heuristic}) works well. For example, doing so helps us to always beat the lower bound (which does not happen otherwise). For details about the exact implementation, please refer to the attached code.

\section{A Greedy approximation of {\PF}}

The next question to consider is whether there is a simpler classifier that achieves similar guarantees to those given by the {\PF} classifier. We answer this question in the affirmative by presenting the iterative {\Greedy} algorithm (\cref{alg:greedy}): select the classifier that classifies the most number of data points correctly, allocate it a weight proportional to this number, discard the data points that it classifies correctly, and simply repeat the above procedure until there are no data points left. The randomized classifier {\Greedy} is defined by giving to each of these classifiers a probability proportional to its weight, i.e., the number of data points classified correctly in the corresponding iteration.

 We now show that the {\Greedy} algorithm provides a constant-factor approximation to the guarantee provided by the {\PF} classifier. Note that the first step in the \emph{while} loop of {\Greedy} (Algorithm \ref{alg:greedy}) involves using the {\ERM} agnostic learning black box. 
\begin{restatable}{theorem}{greedy}
\label{thm:greedy}
If $h_{\G}$ is the {\Greedy} classifier, then for any subset $\mathcal{S} \subseteq \mathcal{N}$ that admits a perfect classifier $h \in \mathcal{H}$, we have $u_S(h_{\G}) \ge \frac{\alpha}{2} + \frac{1}{2 \alpha} \left(\max\left(0, 2\alpha - 1 \right) \right)^2  \ge \frac{\alpha}{2}$, where $\frac{|S|}{n} = \alpha$.
\end{restatable}

\begin{proof}
Consider any set $\mathcal{S}$ of size $\alpha n$ (where $n = |\mathcal{N}|$) that admits a perfect classifier $h_\mathcal{S}$. Let $h_1^*$ be the classifier found at the first step of {\Greedy} (Algorithm \ref{alg:greedy}). Suppose it classifies $a n$ points in $\mathcal{S}$ correctly, $b n$ points in $S$ incorrectly, and $c n$ points not in $\mathcal{S}$ correctly. Clearly, $a + b = \alpha$, and $c \le 1 - \alpha$. Further, the greedy choice implies $b \le c$, so that $b \le 1 - \alpha$. Therefore, $a = \alpha - b \ge \max(2 \alpha - 1, 0)$.

Since {\Greedy} assigns $p_1 = a + d$ and classifies $a n$ points in $S$ correctly, the total utility generated on $\mathcal{S}$ is $a (a+c) n \ge a (a+b)n$.

At the second step, $|T_2| \ge n b$, since $h_S$ classifies $b n $ points correctly. Suppose $|T_2 \cap \mathcal{S}| = y_1 n$. Similarly, $|T_3| \ge (b-y_1) n$; let $|T_3 \cap \mathcal{S}| = y_2 n$, and so on. Therefore, the total utility generated by {\Greedy} is at least
\begin{align*}
    n \cdot \left( a (a+b) + \sum_{q \ge 1} (b - z_{q-1}) y_q \right), 
\end{align*}
where $z_q = \sum_{t = 1}^q y_t$ and $\sum_{q \ge 1} y_q = b$.

Now focus on the term $\sum_{q \ge 1} (b - z_{q-1}) y_q$.
\begin{align*}
    \sum_{q \ge 1} (b - z_{q-1}) y_q &= \sum_{q \ge 1} (b - y_{q-1}) (z_q - z_{q-1}) \\
     &\ge \sum_{q \ge 1} (b - z_{q-1}) (z_q - z_{q-1}) \\
     &\ge \int_{0}^b (b-x) \mathrm{d} x.
\end{align*}

Therefore,
\begin{align}
    &n \cdot \left( a (a+b) + \sum_{q \ge 1} (b - z_{q-1}) y_q \right) \nonumber\\
    &\geq n \cdot \left( a (a+b) + \int_{0}^b x \mathrm{d} x \right), \label{eq:greedybound}
\end{align}
which on further simplification yields
\begin{align*}
    n \cdot \left( a (a+b) + \frac{b^2}{2} \right) = n \cdot \left( \frac{a^2}{2} + \frac{(a+b)^2}{2} \right).
\end{align*}
Now, since $|S| = \alpha n$, $a+b = \alpha$, and $a  \ge \max(2 \alpha - 1, 0)$,
\begin{align*}
    u_S(h_\textsc{G}) &\ge \frac{1}{\alpha} \left( \frac{(a+b)^2}{2}  +  \frac{a^2}{2} \right) \\
    &\ge  \frac{\alpha}{2} +  \frac{1}{2 \alpha} \left(\max\left(0, 2\alpha - 1 \right) \right)^2. \qedhere
\end{align*}
\end{proof}

\begin{figure}
\begin{minipage}{1\textwidth}
\begin{algorithm}[H]
   \caption{{\Greedy} Classifier}
   \label{alg:greedy}
\begin{algorithmic}
    \STATE {\bfseries Input:} data set $\mathcal{N}$, family of classifiers  $\mathcal{H}$
    \STATE Initialize: $r \gets 1$, $\mathcal{S}_r \gets \mathcal{N}$
    \WHILE{$|\mathcal{S}_r| > 0$}
        \STATE $h^*_r \gets \mbox{argmax}_{h_j} u_{\mathcal{S}_r}(h_j)$
        \STATE $\mathcal{T}_r \gets \{i \in \mathcal{S}_r \mid u_i(h^*_r) = 1 \}$
        \STATE $p_r \gets \frac{|\mathcal{T}_r|}{n}$
        \STATE $\mathcal{S}_{r+1} \gets \mathcal{S}_r \setminus \mathcal{T}_r$
        \STATE $r \gets r+1$
    \ENDWHILE
    
    \STATE {\bfseries Return:} classifier $h_{\G}$ that chooses $h_s^* \in \mathcal{H}$ with probability $p_s$ for $s = 1,2,\ldots,r-1$
\end{algorithmic}
\end{algorithm}
\end{minipage}
\end{figure}  

Therefore, for any set $\mathcal{S}$ that has a perfect classifier, the average utility of {\Greedy} is at least half the average utility of Proportional Fairness; for large sets, the approximation factor is better, and approaches $1$ as $\alpha \rightarrow 1$. We complement our $\frac{\alpha}{2}$ analysis with the example below which shows that the above bound is tight when $\alpha \to 0$.

\begin{example}\label{ex:greedy}
 Let $\mathcal{N}$ consist of the following $n = \frac{k(k+1)}{2}$ data-points:
\[
\begin{matrix}
(1, 1), & (1, 2), & \ldots, & (1, k - 1), & (1, k),\\
(2, 1), & (2, 2), & \ldots, & (2, k - 1), &\\
\ldots, & \ldots, & \ldots, &\\
(k - 1, 1), & (k - 1, 2), &\\
(k, 1). &
\end{matrix}
\]
There are $k + 1$ classifiers:
\begin{enumerate}
    \item $h_1$ correctly classifies every $(1, i)$ for $i \in [k]$.
    \item The classifier $h_j$ ($j = 2, 3, \ldots, k$) correctly classifies every $(j, i)$ for $i \in [k - j + 1]$ as well as $(1, j)$.
    \item The classifier $h_{k+1}$ only correctly classifies $(1, 1)$.
\end{enumerate}

Let $\mathcal{S} = \{(1, i) \mid i \in [k]\}$, i.e., the set which $h_1$ correctly classifies. In the $j$-th round, {\Greedy} can pick $h_{j + 1}$ since it covers $k - j + 1$ new data-points (tied with $h_1$). However, each of $h_2, h_3, \ldots, h_{k+1}$ only covers one data-point in $S$, meaning that the accuracy of {\Greedy} on $S$ is only $\frac{1}{k} = \frac{\alpha}{2} \cdot \frac{k + 1}{k}$, where $\alpha = \frac{|\mathcal{S}|}{|\mathcal{N}|}$.
\end{example}

\section{Generalization bound for {\PF}}
In this section, we will outline how to derive generalization bounds for {\PF}. To begin with, assume that $\mathcal{H}$ is a hypothesis space of finite VC dimension, say $d$. Let $\Delta(\mathcal{H})$ be the space of randomized classifiers. Let $D$ represent the distribution of data.

For some $h \in \mathcal{H}$, as before, let $u_i(h)$ be the utility of a point $i$ from classifier $h$. Let $S_i(h) = \log(u_i(h) + \eps)$, where $\eps > 0$. Let $A_{\Delta(\mathcal{H})}$ be the set of (expected) outcomes for $i$ in a finite sample $S$ of size $m$ based on the function $S_i(\cdot)$ and the space $\Delta(\mathcal{H})$. From standard generalization bounds (Theorem 26.3 in \citet{shalev2014understanding}) based on Rademacher complexity, we have a bound of $\frac{2R}{\delta}$ with probability $1-\delta$ over the choice of $S$ (drawn randomly from $D$), where $R$ is the Rademacher complexity of $A_{\Delta(\mathcal{H})}$.

We now show how to bound $R$. First note that in the domain $[\eps, +\infty)$, $\log$ is $\frac{1}{\eps}$-Lipschitz. By Lemma 26.9 in \citet{shalev2014understanding}, we need only bound the Rademacher complexity corresponding to $u_i(h)$ for $h \in \Delta(\mathcal{H})$, and $R$ will be at most factor $\frac{1}{\eps}$ of this bound. Moreover, by Lemma 26.7 of \citet{shalev2014understanding}, the Rademacher complexity is unaffected by taking the convex hull of a set, and therefore, we need only consider $h \in \mathcal{H}$. Since $\mathcal{H}$ is of finite VC dimension $d$, we can bound the resulting Rademacher complexity by $O(\sqrt{d/m})$ \citep{bartlett2002rademacher}. 

For example, if $\mathcal{H}$ consisted of linear classifiers whose coefficients have a bounded $\ell_2$ norm of $1$, and the features $x$ are similarly bounded, then the $R$ is $O(\frac{1}{\eps \sqrt{m}})$. Therefore, the generalization error is bounded by $O(\frac{1}{\delta \eps \sqrt{m}})$. And if this is at most $\eps$, then we get a $2\eps$-approximation of the optimal {\PF} solution, i.e., for $m = O(\frac{1}{\eps^4 \delta^2})$.



\section{Omitted details from the experiments}\label{app:experiments}
In all our experiments, we drop sensitive features such as race and gender. For all the methods in the paper, whenever possible, we use standard implementations from the \texttt{scikit-learn} (0.21.3) package \citep{scikit-learn}. Wherever we need an agnostic learning black-box, we just use Logistic Regression, since it performs reasonably well on the data-sets explored in this paper. For the \texttt{adult} data set, we use the train data and test data as available at \url{https://archive.ics.uci.edu/ml/datasets/Adult}. For the \texttt{compas} data set, we do a train-test split of 80-20. We use the same features as discussed in \url{https://github.com/propublica/compas-analysis}. All code to compute the various classifiers is attached herewith.

Below, we provide some additional plots.

\subsection{Additional plots for the \texttt{compas} data-set}
We look at subsets that are mis-classified to varying degrees by {\LR} and {\Ada}, and check how well {\PF} and {\Greedy} perform on them. More precisely, for each $x \in \{0,0.1,\ldots, 1\}$, we sample 25\% of the test set such that a fraction $x$ of it comes from the points in the test set that are classified correctly by {\LR} and {\Ada} respectively. For example, if $x = 0.2$, then one-fifth of the sampled subset comes from the correctly classified points, and four-fifths from the incorrectly classified. \Cref{fig:compas_pf-vs-erm_25} shows the accuracy obtained by the different methods averaged across 100 samples of subsets obtained (as described above) based on {\LR} and {\Ada}. We see that the performance of \emph{both {\PF} and {\Greedy} is similar irrespective of whether the subsets are defined based on {\LR} or {\Ada}}. Similar trends are seen when other methods are used instead of {\LR} or {\Ada}.

In \Cref{fig:compas_pf-vs-erm_75}, we repeat the procedure used in \Cref{fig:compas_pf-vs-erm_25}, but this time the sampled subset is $75\%$ of the test set. More precisely, for each $x$, we sample 75\% of the test set such that a fraction $x$ of it comes from the points in the test set that are classified correctly by {\LR} and {\Ada} respectively. Since $75\%$ is a large subset, we see both {\PF} and {\Greedy} doing a bit worse than the ERM methods, and also getting close to the lower bound when the subset becomes perfectly classifiable. 

\begin{figure*}[ht!]\centering
\begin{subfigure}{0.48\linewidth}\centering
\includegraphics[scale=0.40]{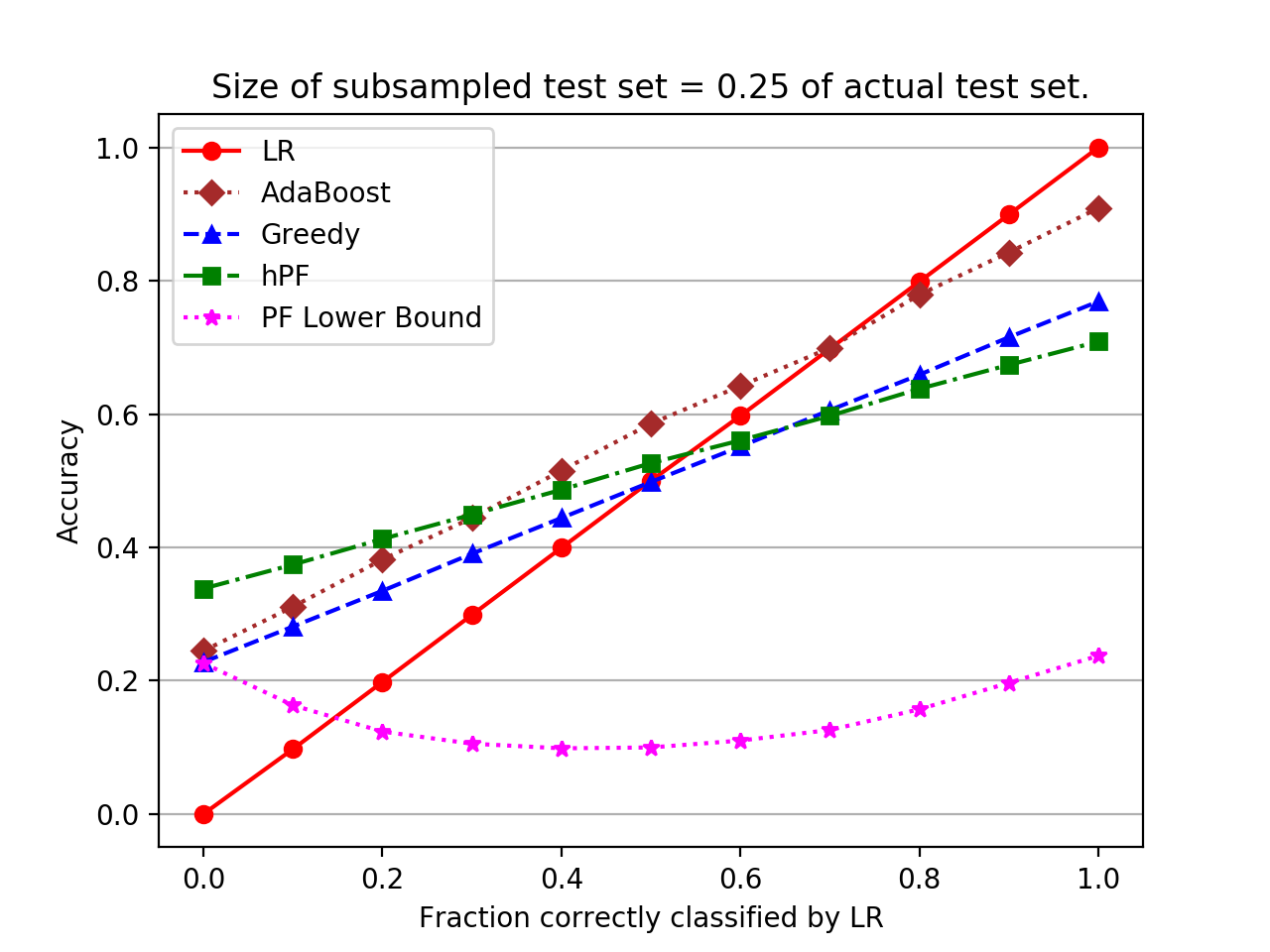}
\end{subfigure}%
\quad
\begin{subfigure}{0.48\linewidth}\centering
\includegraphics[scale=0.40]{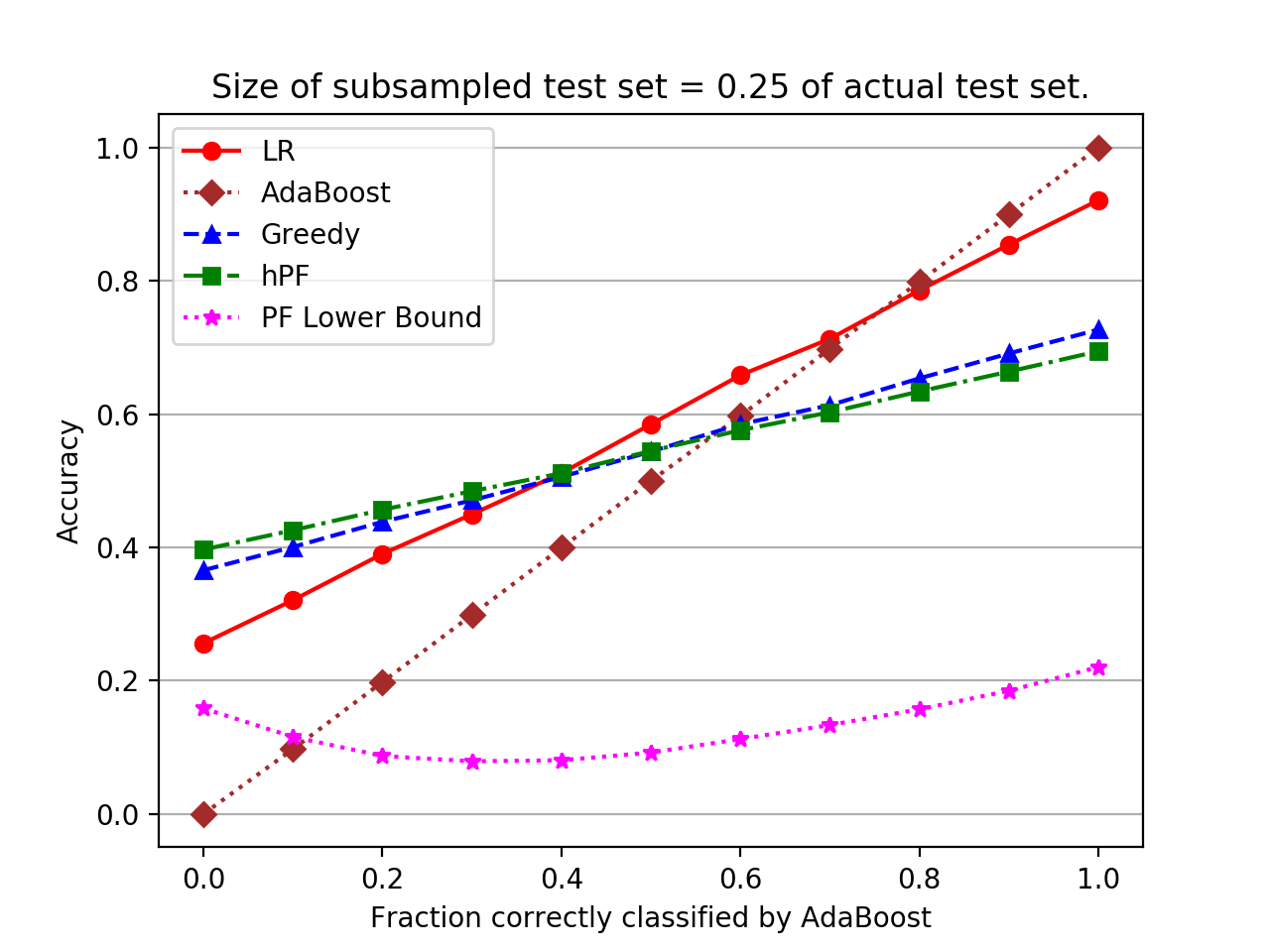}
\end{subfigure}
\caption{Accuracy on sampled subsets which are 25\% of the test set (\texttt{compas}): $x$ axis denotes the fraction of points that are classified correctly by {\LR} (left) and {\Ada} (right).}
\label{fig:compas_pf-vs-erm_25}
\end{figure*}

\begin{figure*}[ht!]\centering
\begin{subfigure}{0.48\linewidth}\centering
\includegraphics[scale=0.40]{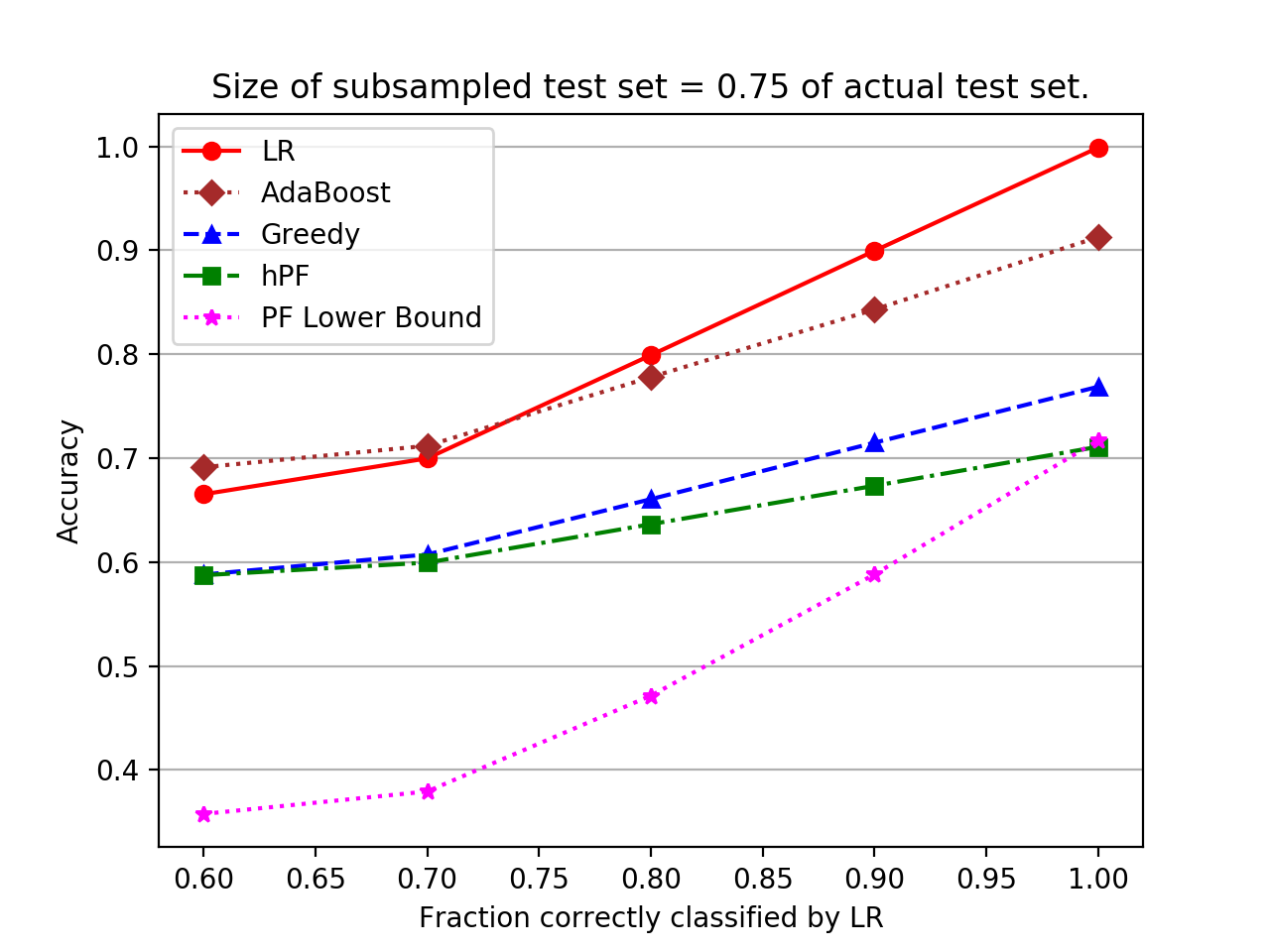}
\end{subfigure}%
\quad
\begin{subfigure}{0.48\linewidth}\centering
\includegraphics[scale=0.40]{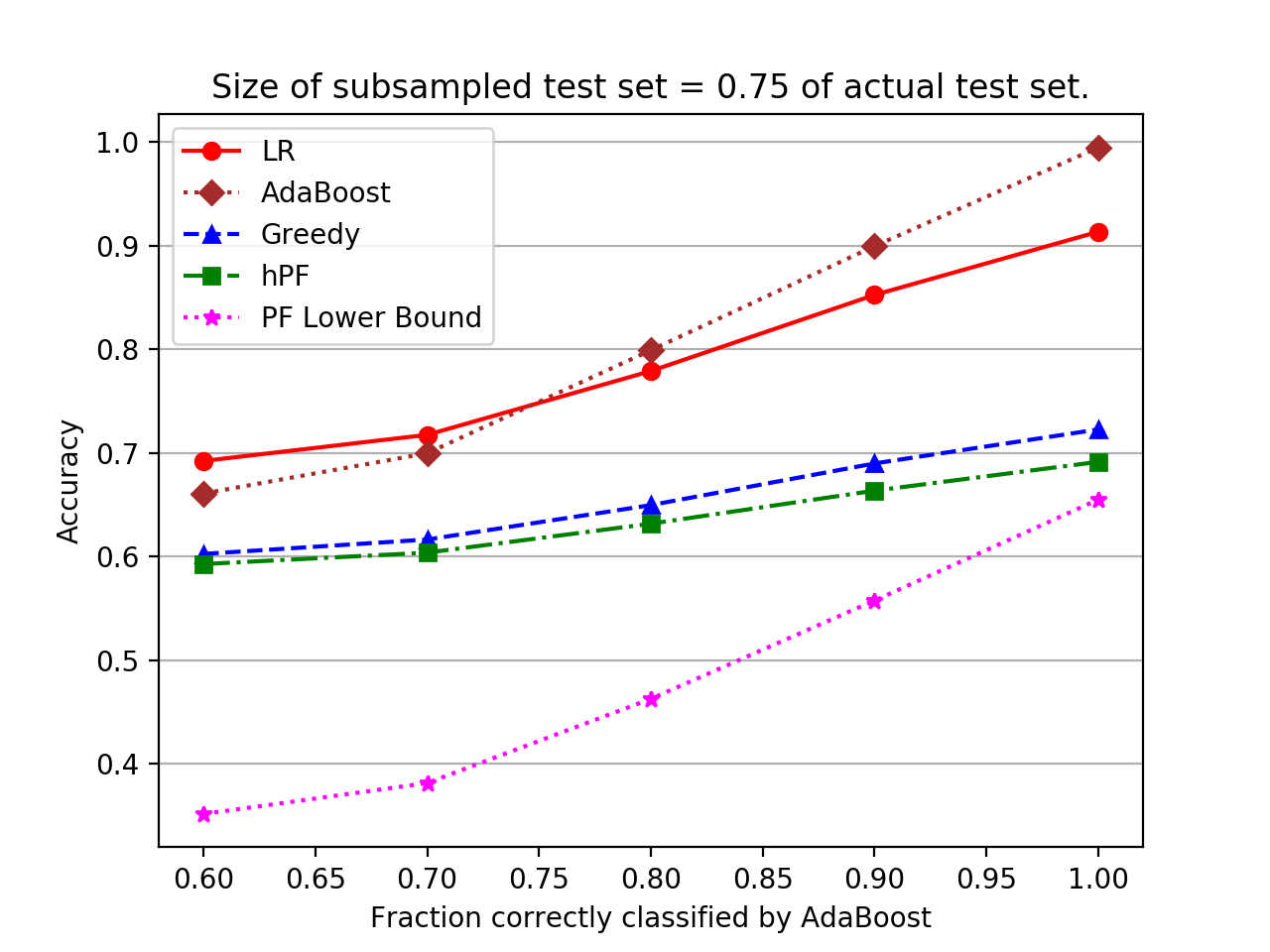}
\end{subfigure}
\caption{Accuracy on sampled subsets which are 75\% of the test set (\texttt{compas}): $x$ axis denotes the fraction of points that are classified correctly by {\LR} (left) and {\Ada} (right).}
\label{fig:compas_pf-vs-erm_75}
\end{figure*}

Similar observations can be made on the \texttt{adult} data set.

\end{document}